\documentclass[a4paper]{article}[12pts]

\usepackage[english]{babel}
\usepackage[utf8x]{inputenc}
\usepackage[T1]{fontenc}

\normalsize

\usepackage[a4paper,top=1in,bottom=1in,left=1in,right=1in,marginparwidth=1in]{geometry}

\usepackage{natbib}
\usepackage{setspace}
\usepackage{amsmath}
\usepackage{amssymb} 
\usepackage{amsthm}
\usepackage{amsfonts, mathtools}
\usepackage{algorithm,algpseudocode}
\usepackage{bm}
\usepackage{bbm}
\usepackage{comment}
\usepackage{graphicx}
\usepackage{multirow, booktabs}
\usepackage{caption}
\usepackage{subcaption}
\usepackage[colorinlistoftodos]{todonotes}
\usepackage{hyperref}
\hypersetup{colorlinks=true, allcolors=blue}

\usepackage{multirow}
\usepackage{textgreek}
\usepackage{adjustbox}
\usepackage{mdframed}

\definecolor{azure}{rgb}{0.9, 0.95, 1.0}

\theoremstyle{plain}
\newtheorem{theorem}{Theorem}[section]
\newtheorem{proposition}[theorem]{Proposition}
\newtheorem{lemma}[theorem]{Lemma}
\newtheorem{corollary}[theorem]{Corollary}
\theoremstyle{definition}
\newtheorem{definition}[theorem]{Definition}
\newtheorem{assumption}[theorem]{Assumption}
\newtheorem{example}[theorem]{Example}
\theoremstyle{remark}

\renewcommand{\thefootnote}{\fnsymbol{footnote}}

\DeclareMathOperator*{\argmin}{arg\,min}

\title{Reward Modeling with Ordinal Feedback: Wisdom of the Crowd}
\author{Shang Liu$^{*,1}$, Yu Pan$^{*,2}$, Guanting Chen$^{3}$, Xiaocheng Li$^{1}$}
\date{\small
Imperial College Business School, Imperial College London$^{1}$ \\  Department of Intelligent Transportation, HKUST-GZ$^{2}$ \\  
Department of Statistics and Operations Research, University of North Carolina$^{3}$
}

\begin{document}
\maketitle
\onehalfspacing

\def\thefootnote{*}\relax\footnotetext{Equal contribution. Corresponding to Shang Liu (s.liu21@imperial.ac.uk) and Yu Pan (yupan@hkust-gz.edu.cn). The authors thank Hengzhi He and Zhongze Cai for providing helpful discussions and improvements to the manuscript.}

\begin{abstract}
Learning a reward model (RM) from human preferences has been an important component in aligning large language models (LLMs). The canonical setup of learning RMs from pairwise preference data is rooted in the classic Bradley-Terry (BT) model that accepts binary feedback, i.e., the label being either Response 1 is better than Response 2, or the opposite. Such a setup inevitably discards potentially useful samples (such as ``tied'' between the two responses) and loses more fine-grained information  (such as ``slightly better''). In this paper, we propose a framework for learning RMs under \textit{ordinal feedback} which generalizes the case of binary preference feedback to any arbitrary granularity. Specifically, we first identify a marginal unbiasedness condition, which generalizes the assumption of the BT model in the existing binary feedback setting. The condition validates itself via the sociological concept of the wisdom of the crowd. Under the condition, we develop a natural probability model for pairwise preference data under ordinal feedback and analyze its properties. We prove the statistical benefits of ordinal feedback in terms of reducing the Rademacher complexity compared to the case of binary feedback. The proposed learning objective and the theory also extend to hinge loss and direct policy optimization (DPO). In particular, the theoretical analysis may be of independent interest when applying to a seemingly unrelated problem of knowledge distillation to interpret the bias-variance trade-off therein. The framework also sheds light on writing guidance for human annotators. Our numerical experiments validate that fine-grained feedback leads to better reward learning for both in-distribution and out-of-distribution settings. Further experiments show that incorporating a certain proportion of samples with tied preference boosts RM learning.
\end{abstract}

\section{Introduction}
\label{sec:intro}

Reinforcement learning from human feedback (RLHF) \citep{christiano2017deep, ziegler2019fine, askell2021general, ouyang2022training} is vital to aligning large language models (LLMs) with human preferences. The RLHF involves either explicitly training a reward model (RM) from human preferences data \citep{ouyang2022training} or implicitly using the LLM itself as one \citep{rafailov2024direct}. However, there is an inconsistency between current ways of collecting human preference data and the training of reward models. For example, the Llama team collects fine-grained human feedback: they not only collect the preferred response but also 4 levels named ``significantly better'', ``better'', ``slightly better'', and ``marginally better'' \citep{llama2024llama}, while the post-training of Llama 3 treats ``significantly better'' and ``better'' as the same and discard all the others. Such a process wastes the potentially useful samples that cost the human annotators additional time and also it may loss the useful information hidden in the preference level.

In this paper, we study the problem of reward modeling under ordinal feedback. Specifically, we relate the annotator's preference feedback with the probability that one response is better than the other on a population level. We introduce a marginal unbiasedness assumption as the only assumption that validates the probability setup of the ordinal feedback system. The assumption is rooted in the sociological concept of the wisdom of the crowd. Under the assumption, we analyze the properties of the probability model of ordinal feedback. We propose a learning objective for reward modeling with ordinal feedback, and the objective function naturally generalizes the case of binary feedback. Theoretically, we establish the advantage of ordinal feedback, which draws an interesting connection with the literature on soft labeling and knowledge distillation.

Our paper is organized as follows:
\begin{enumerate}
\item In Section \ref{sec:setup}, we model the general ordinal feedback by relating the feedback with the probability that a certain response is better than the other on a population level. The binary feedback ($\mathcal{Z} = \{0, 1\}$) is extended to the general \textit{ordinal feedback} ($\mathcal{Z} = \{z_j\}_{j=1}^m$ for $0 \leq z_1 < \dots < z_m \leq 1$), providing a way to transform the qualitative label into quantitative ones.
\item In Section \ref{sec:probability_model}, we build up the probability model of ordinal feedback. We first set the oracle probability as the standard preference model and present the only assumption (Assumption \ref{assm:wisdom_of_crowd}) that the annotators in the marginal sense are giving an unbiased estimation of that oracle. Such an assumption (which we call ``wisdom of the crowd'') is only a generalization of the binary case that regards the feedback as a Bernoulli random variable. In this light, we suggest revising the annotation guideline by providing a direct quantitative description of the qualitative opinions. Furthermore, under the assumption, we prove the existence and the uniqueness (up to convex combinations) of the ordinal feedback.
\item In Section \ref{sec:benefits_of_ord}, we prove the statistical benefits of the ordinal feedback. More specifically, the Rademacher complexity is reduced if the loss function satisfies the affinity condition (which is fulfilled by the common cross-entropy loss). Such a conclusion also holds for direct policy optimization (DPO). The result is proved via a special coupling argument that we call \textit{hierarchical expectation}, which also provides a new bias-variance trade-off in knowledge distillation and soft labeling.
\item In Section \ref{sec:experiments}, we conduct two numerical experiments. The first experiment sets up four different ordinal feedback systems (oracle, 5-level, 3-level, and binary) and validates the theoretical findings that fine-grained ordinal feedback benefits RM training by achieving higher accuracies in both in-distribution (ID) and out-of-distribution (OOD) settings. The second experiment mixes the training data with a proportion of tied and untied samples. With the same number of training samples, we find out that a certain level of tied samples boosts RM learning.
\end{enumerate}

\section{Problem Setup}
\label{sec:setup}

Consider the task of reward modeling based on the pairwise preference data. Each data sample consists of a tuple 
$$(x, y_1, y_2, z)$$
where $x\in \mathcal{X}$ denotes a prompt, $y_1, y_2\in \mathcal{Y}$ are two candidate responses to the prompt $x$, and $Z$ is a random variable (taking values in $\mathcal{Z}$) that denotes the feedback (generated by either human annotators or advanced AI models) indicating the preference between $y_1$ and $y_2$. The feedback $Z$ can be viewed as a proxy of the probability that $y_1$ is better than $y_2$ for the prompt $x$, denoted by $\mathbb{P}(y_1 \succ y_2 | x)$.

The task of reward modeling thus refers to the learning of a reward function $r_{\theta}(x,y): \mathcal{X}\times \mathcal{Y}\rightarrow \mathbb{R}$ with parameter $\theta\in \Theta$ from 
an annotated dataset $$\mathcal{D}_{\mathcal{Z}} \coloneqq \{(x_i, y_{i,1}, y_{i, 2}, Z_i)\}_{i=1}^n.$$
The prevalent way to relate the reward model with the preference probability is via the Bradley-Terry model \citep{bradley1952rank}
\[
\mathbb{P}(y_1 \succ y_2 | x) \approx \frac{\exp\big(r_{\theta}(x, y_1)\big)}{\exp\big(r_{\theta}(x, y_1)\big) + \exp\big(r_{\theta}(x, y_2)\big)}
\]
where we approximate the probability with the softmax reward values on the right-hand side.

\textbf{Binary feedback}: In the canonical setup (\cite{bai2022training, ouyang2022training} among others), the feedback $Z$ takes binary values, i.e., $\mathcal{Z} = \{0,1\}$. Here one assumes $Z_i$ is a Bernoulli random variable such that
\begin{equation}
\mathbb{P}(Z_i=1) = 1-\mathbb{P}(Z_i=0) = \mathbb{P}\left(y_{i,1} \succ y_{i,2} | x_i\right).
\label{eqn:Z_binary}
\end{equation}
This assumption has been the backbone of the training of many mainstream reward models. 

\textbf{Ordinal feedback}: In this paper, we consider the setting of \textit{ordinal feedback} which gives a richer feedback structure than the binary feedback above and is defined as follows.

\begin{definition}[Ordinal Feedback]
\label{def:ord_fb}
Suppose the feedback $Z$ takes values in $\mathcal{Z} \coloneqq \{z_1, \dots, z_m\}$ where $0\le z_1 < \cdots < z_m\le 1$, we call $Z$ an ordinal feedback and $\mathcal{Z}$ the ordinal feedback set.
\end{definition}

The binary feedback is a special case of the ordinal feedback by letting $m=2$, $z_1=0$, and $z_2=1.$ The motivation for us to introduce this ordinal feedback definition is to capture the richer annotation feedback options in practice where the annotator is allowed to choose from 
$$\mathcal{Z}_{\text{text}} = \{\text{better than, same as, worse than}\}$$
$$\mathcal{Z}_{\text{text}} = \{\text{better than, slightly better, slightly worse, worse than}\}$$
to describe the preference between responses $y_1$ and $y_2.$ We defer to the next section the questions of how to match $\mathcal{Z}_{\text{text}}$ with $\mathcal{Z}$ and how to determine the values of $z_1,...,z_m$ in $\mathcal{Z}$. 

Suppose for now, we have the dataset $\mathcal{D}$ where $Z_i$'s take values in $\mathcal{Z}$. We propose to learn the reward model by minimizing the following objective function
\begin{equation}
\min_{\theta} \ \sum_{i=1}^n - Z_{i}\cdot \log\left(\sigma\left(r_{\theta}(x_i, y_{i,1}) - r_{\theta}(x_i, y_{i,2})\right)\right) - (1-Z_{i})\cdot \log\left(\sigma\left(r_{\theta}(x_i, y_{i,2}) - r_{\theta}(x_i, y_{i,1})\right)\right), 
\label{eqn:learning_objective}
\end{equation}
where $\sigma(\cdot)$ is the sigmoid function such that $\sigma(x) = \exp(x)/(1+\exp(x))$.

When the feedback is binary $Z_i\in\{0,1\}$, the above objective function reduces to exactly how people learn the reward models under the Bradley-Terry assumption. When the feedback is soft and takes value in $[0,1]$, we will demonstrate how such a loss function characterizes the ordinal feedback options of $\mathcal{Z}_{\text{text}}.$ 

Intuitively, this objective function better utilizes the annotated data and avoids the shortcomings of two alternative heuristics: (i) Discarding all the samples annotated as ``same as'' for that these samples do not integrate with the binary feedback Bradley-Terry model; $\rightarrow$ this causes loss of samples. (ii) Combining ``better than'' and ``slightly better'' as better (and the same for worse); $\rightarrow$ this causes loss of information.

This shift from binary to ordinal as above, though natural, raises three questions: 
\begin{itemize}
\item[-] Probability: The binary feedback is rooted in the Assumption \eqref{eqn:Z_binary}. What is the underlying probability model to connect $\mathcal{Z}$ and $\mathcal{Z}_{\text{text}}$ with the probability $\mathbb{P}\left(y_{1} \succ y_{2} | x_i\right)$?
\item[-] Annotation: How does the thinking process of human annotators relate to the probability model of ordinal feedback from a social choice perspective?
\item[-] Learning: How does this ordinal feedback setup affect the learning of the reward model from a machine-learning perspective?
\end{itemize}

In Section \ref{sec:probability_model}, we address the first two questions, and in Section \ref{sec:benefits_of_ord}, we address the last question. In Section \ref{sec:experiments}, we present numerical experiments. For the main paper, we study the reward modeling under the cross-entropy loss as the objective function above, and in Appendix \ref{apd:hinge_experiments}, we discuss how the ordinal feedback setup can be addressed under hinge loss.

\section{Probability Model of Ordinal Feedback}
\label{sec:probability_model}

We first define the oracle feedback model as 
$$z_{{\text{oracle}}}(x,y_1,y_2) \coloneqq \mathbb{P}\left(y_{1} \succ y_{2} | x\right).$$
This is \textit{the} preference model that one aims to learn, regardless of whether assuming the Bradley-Terry model or whatever other preference/reward model. Here we should think the probability space being the whole population that use the language, and a human annotator as a random draw from the population. 

\begin{assumption}[Ordinal feedback probability model -- wisdom of the crowd]
\label{assm:wisdom_of_crowd}
We assume the ordinal feedback $Z$ defined in Definition \ref{def:ord_fb} satisfies
\[
\mathbb{E}[Z | (x, y_1, y_2)] = z_{\text{oracle}}(x, y_1, y_2) \quad \text{for any } (x, y_1, y_2) \in \mathcal{X} \times \mathcal{Y}^2.
\]
\end{assumption}

The assumption is the only one we make for the ordinal feedback model. As we will see, this one assumption alone is sufficient to define the probability model of the ordinal feedback setting and validate the learning of the reward model. To interpret the assumption, we emphasize that it is not stricter than the existing assumption people impose for the binary feedback model. Specifically, under the binary feedback model where $Z$ takes values in $\mathcal{Z}=\{0,1\}$, Assumption \ref{assm:wisdom_of_crowd} is equivalent to the Assumption \eqref{eqn:Z_binary}. 

To see another example, consider the set $\mathcal{Z}=\{0,0.5,1\}$  where the labels $0$ and $1$ denote ``better'' and ``worse'' respectively, and the label $0.5$ denotes ``same as''. The assumption then requires that with $Z$ taking values in this new $\mathcal{Z}$ with an additional $0.5$ option, its expectation matches the oracle value $z_{\text{oracle}}$ on the sample. Under the general ordinal feedback setting, human annotators label their preferences on different scales, i.e., 
$$\mathcal{Z}=\{0,1\}, \ \{0,0.5,1\}, \text{ or } \{0,0.25,0.5,0.75,1\}.$$
The assumption says that the change of scales does not introduce bias that twists the oracle preference on the population level.

\textbf{Sociological interpretation}: We name the assumption by ``wisdom of the crowd''; the concept was first coined by the article \textit{Vox Populi} \citep{galton1907vox} for a social experiment under the title ``the voice of the people''. The social experiment is about a weight-judging competition conducted in England for random people to guess the weight of an ox. The average of all 787 guesses was 1,197 pounds, while the actual weight was 1,198 pounds, as shown in Figure \ref{fig:ox_weight_and_crowd_feedback}. Each individual's guess can be far off the target yet the population average tends to be very accurate. For the context of preference annotation, Assumption \ref{assm:wisdom_of_crowd} and the current practice of human annotation exercise the wisdom of the crowd in two folds: First, each individual annotator has no access to the population preference $z_{\text{oracle}}$, but their annotation can be viewed as an unbiased random realization of $z_{\text{oracle}}$. Second, such unbiasedness does not change if a difference annotation scale $\mathcal{Z}$ (feedback set) is used. 




\begin{figure}[ht!]
\centering
\includegraphics[width=\linewidth]{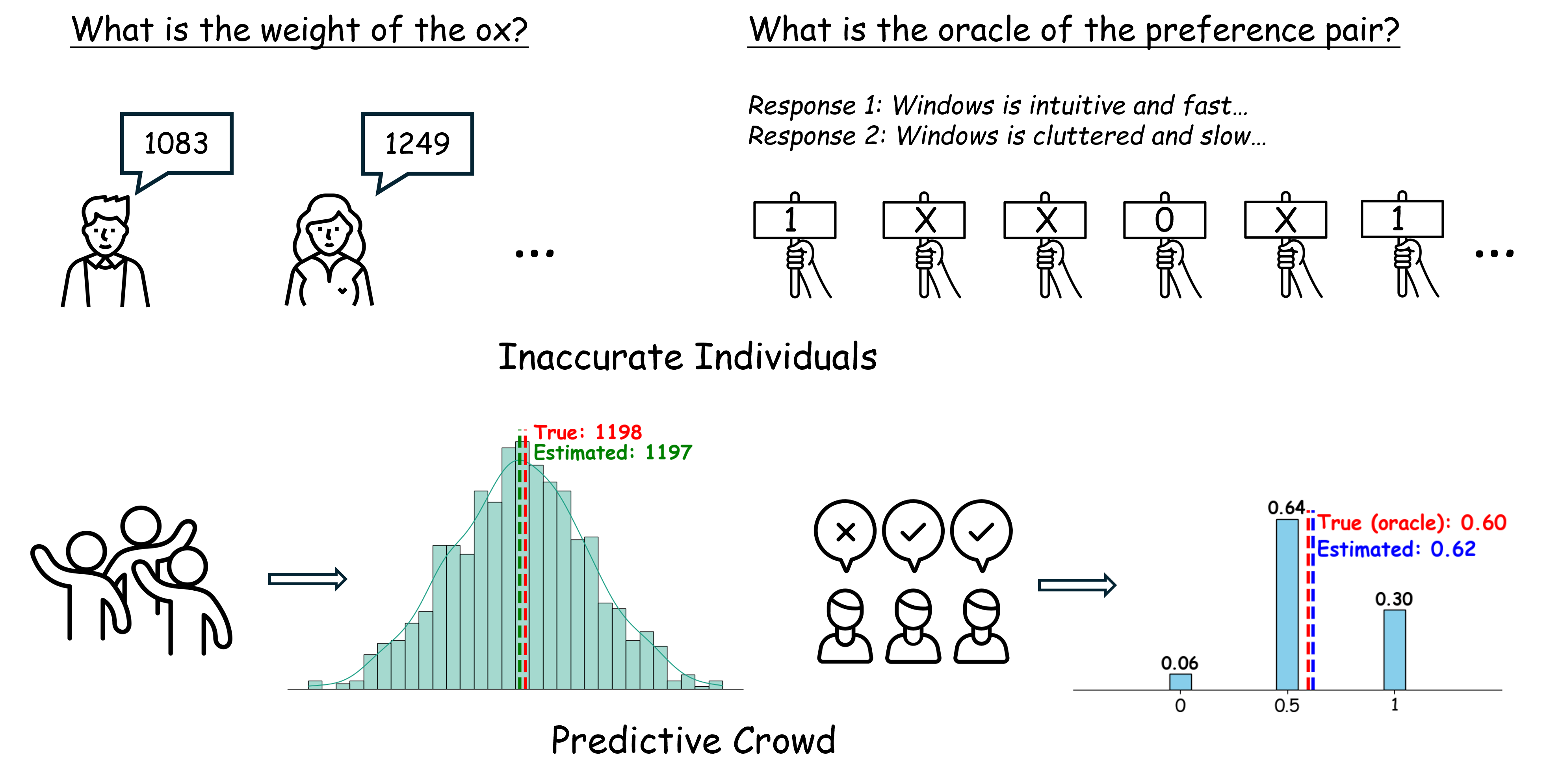}
\caption{\small Wisdom of the crowd. Left: Each individual guess can be far off the target for an ox-weight-guessing social experiment, but the average tends to be very accurate. Each human annotator has not access to the population oracle preference model $z_{\text{oracle}}$, but their annotation constitutes an unbiased realization of $z_{\text{oracle}}$.}
\label{fig:ox_weight_and_crowd_feedback}
\end{figure}

\subsection{Implications on annotation guidance}

In practice, annotators label in the set $\mathcal{Z}_{\text{text}}$ (e.g. $\{\text{better, same as, worse}\}$), and this results in a gap towards the set $\mathcal{Z}=\{z_1,...,z_m\}$ used in the learning of the reward model $\eqref{eqn:learning_objective}$. For the binary or 3-level feedback setting, the following conversion is natural and reflects the thinking process of the annotator:
$$\mathcal{Z}_{\text{text}} = \{\text{better than, worse than}\} \Leftrightarrow \mathcal{Z} = \{0,1\},$$
$$\mathcal{Z}_{\text{text}} = \{\text{better than, same as, worse than}\} \Leftrightarrow \mathcal{Z} = \{0, 0.5, 1\}.$$
For a more fine-grained 5-level feedback setting, it is less clear whether one can convert as follows, 
$$\mathcal{Z}_{\text{text}} = \{\text{better than, slightly better, same, slightly worse, worse than}\} \Leftrightarrow \mathcal{Z} = \{0, 0.25, 0.5, 0.75, 1\}.$$

In this light, our ordinal feedback model and Assumption \ref{assm:wisdom_of_crowd} provide the following insights in guiding the annotations. Rather than to provide vague wording of ``better than'' or ``sightly better'', one can write the following in the guidance to the human annotators:

\begin{mdframed}[linewidth=1.5pt,linecolor=azure,backgroundcolor=azure]
\textbf{Annotation guideline:} The label ``slightly better'' represents that 75\% of the population think response $y_1$ is better than response $y_2.$ The label ``slightly worse'' represents that 25\% of the population think response $y_1$ is better than response $y_2.$ 
\end{mdframed}

Such an additional guidance endows the labels in $\mathcal{Z}_{\text{text}}$ a numerical meaning. The corresponding numerical values can be directly used in the learning objective \eqref{eqn:learning_objective}.

\subsection{Universal existence of feedback probability model}

In this subsection, we detour from the discussions of reward modeling and show the universal existence of probability models that satisfy Assumption \ref{assm:wisdom_of_crowd}, which ensures the assumption is well-defined. The following theorem states that for any ordinal feedback set $\mathcal{Z}$ and any oracle model $z_{\text{oracle}},$ there exists a probability model satisfying Assumption \ref{assm:wisdom_of_crowd}.
\begin{theorem}
\label{thm:unbiased_interpolation}
For any ordinal feedback set $\mathcal{Z} = \{z_1, \dots, z_m\}$ and any oracle model $z_{\text{oracle}}(x, y_1, y_2)$, one can construct an ordinal feedback $Z$ as a random variable that satisfies Assumption \ref{assm:wisdom_of_crowd} in the following way. Specifically, if $z_{\text{oracle}} \in [z_j, z_k]$ for some $j, k \in [m]$, then one can set the marginal probability measure $\mu_{j, k}(z) \coloneqq \mathbb{P}(Z = z | (x, y_1, y_2))$ to be
\[
\mu_{j, k}(z) =
\begin{cases}
(z_k - z_{\text{oracle}})\big/(z_k - z_j), \quad &\text{if }z=z_j,\\
(z_{\text{oracle}} - z_j)\big/(z_k - z_j), \quad &\text{if }z=z_k,\\
0, \quad & \text{otherwise}.
\end{cases}
\]
The ordinal feedback $Z$ fulfills Assumption \ref{assm:wisdom_of_crowd}. 

On the other hand, any ordinal feedback satisfying Assumption \ref{assm:wisdom_of_crowd} must be a convex combination of such constructions. More specifically, for any ordinal feedback $Z$ with marginal probability measure $\mu(z) \coloneqq \mathbb{P}(Z = z | (x, y_1, y_2))$ satisfying Assumption \ref{assm:wisdom_of_crowd}, there must exist non-negative real numbers $\sum_{j,k} \alpha_{j, k} = 1$ such that
\[
\mu = \sum_{j,k} \ \alpha_{j, k} \ \mu_{j, k},
\]
where the summation is made for all $(j, k)$ pairs such that $z_{\text{oracle}} \in [z_j, z_k]$.
\end{theorem}

The theorem not only verifies the existence of a probability distribution that satisfies Assumption \ref{assm:wisdom_of_crowd}, but also provides a full characterization of all the distributions
that satisfy the assumption. It says that all such distributions should be (convex combinations of) some two-point distributions, where the weights are assigned to interpolate the linear system
\[
\begin{cases}
\mu_{j,k}(z_j) \cdot z_j + \mu_{j,k}(z_k) \cdot z_k = z_{\text{oracle}},\\
\mu_{j,k}(z_j) + \mu_{j,k}(z_k) = 1.
\end{cases}
\]
For example, consider a reward model with ties where $\mathcal{Z} = \{0, 0.5, 1\}$. If the oracle feedback is, say, $0.8$, then we can construct a feedback $Z$ with probability masses $\mu(0.5) = 0.4$ and $\mu(1) = 0.6$ such that the unbiased assumption is fulfilled. In this way, if we are (in an ideal world) given an oracle model  $z_{\text{oracle}}$, we can generate different unbiased ordinal feedback problem models accordingly.

\section{Statistical Benefits of Ordinal Feedback}

\label{sec:benefits_of_ord}

Intuitively, the more fine-grained ordinal feedback offers more information and should bring more benefits to the learning of a reward model. In this section, we provide a theoretical explanation that pinpoints the benefits, which are not restrictive to the traditional RLHF but also applicable to DPO. The theory not only captures the benefits of the ordinal feedback model but also provides insights into the technique of \textit{soft labeling} in knowledge distillation \citep{ba2014deep, hinton2015distilling, muller2019does, phuong2019towards, yuan2020revisiting, zhou2021rethinking}, which can be of independent interests. In short, the theoretical result says that any feedback model satisfying Assumption \ref{assm:wisdom_of_crowd} brings statistical benefits compared to the canonical binary feedback model.


Before we proceed, we introduce the following definition for the loss function. 

\begin{definition}[Feedback Affinity]
\label{def:affine}
We say a loss function $\ell(Z, z)$ is affine to the feedback random variable $Z$ if it is affine to $Z$ for any $z \in [0, 1]$. In other words,
\begin{equation}
\mathbb{E}_{Z}\big[\ell(Z, z)\big] = \ell\big(\mathbb{E}[Z], z) \text{ for any } z\in[0,1].\label{eqn:loss_affine}
\end{equation}
\end{definition}

Such an affinity condition is not a restrictive one. The two commonly used loss functions, the cross-entropy loss and (a generalized version of) the hinge loss, both satisfy the feedback affinity. We defer the detailed definitions of the loss functions and the proof of the following proposition to Appendix \ref{subapd:loss_affine}.

\begin{proposition}
\label{prop:loss_are_affine}
The cross-entropy loss and the hinge loss satisfy the feedback affinity in Definition \ref{def:affine}.
\end{proposition}


The following proposition says that two ordinal feedback systems lead to the same population loss if both of them satisfy Assumption \ref{assm:wisdom_of_crowd} and the underlying loss function satisfies feedback affinity. 

\begin{proposition}
\label{prop:same_population_loss}
For any two ordinal feedback $Z$ and $Z^\prime$ both satisfying Assumption \ref{assm:wisdom_of_crowd}, if the loss function $\ell$ satisfies the affinity condition \eqref{eqn:loss_affine}, then for any hypothesis class $\mathcal{H}$, we have
\[
\mathbb{E}_{x, y, Z}\Big[\ell\big(Z, h(x, y_1, y_2)\big)\Big] = \mathbb{E}_{x, y, Z^\prime}\Big[\ell\big(Z^\prime, h(x, y_1, y_2)\big)\Big], \quad \forall h \in \mathcal{H}.
\]
\end{proposition}

The function $h$ can be viewed as the reward model function that maps the prompt-responses sample $(x,y_1,y_2)$ into a preference prediction. The proposition implies that any feedback structure under Assumption \ref{assm:wisdom_of_crowd} leads to the same population loss, and thus justifies the assumption from a learning perspective. 

\subsection{Finite-sample benefits}

While the population loss establishes an equivalence between different feedback systems in an asymptotic sense, we illustrate the finite-sample benefits of a more fine-grained ordinal feedback system in the following. We first introduce the concepts of \textit{coupling} and \textit{hierarchical expectation}.


\begin{definition}[Coupling]
\label{def:coupling}
For any two random variables $\xi$ and $\xi^\prime$, if there exist two random variables $\zeta$ and $\zeta^\prime$ over one probability space such that $\zeta$ has the same distribution as $\xi$ and $\zeta^\prime$ the same as $\xi^\prime$, we call them a coupling of $\xi$ and $\xi^\prime$.
\end{definition}
\begin{definition}[Hierarchical Expectation]
\label{def:hierarchical_expectation}
For any two ordinal feedback systems $Z$ and $Z^\prime$ taking values in $\mathcal{Z}$ and $\mathcal{Z}^\prime$ over the same probability space $(\Omega, \mathcal{F}, \mathbb{P})$, if there exists a combination of random variables $(W, W^\prime)$ over a probability space $(\Omega_0, \mathcal{F}_0, \mathbb{P}_0)$ such that
\begin{enumerate}
    \item $(W, W^\prime)$ forms a coupling between $Z$ and $Z^\prime$;
    \item $W = \mathbb{E}[W^\prime | W]$ holds almost surely.
\end{enumerate}
Then we say that $Z$ is a hierarchical expectation of $Z^\prime$.
\end{definition}

The concept of hierarchical expectation defines the relative granularity of the feedback system. In general, if $Z$ is a hierarchical expectation of $Z^\prime$, then we say $Z$ is more fine-grained than $Z^\prime$ since $Z$ allows annotators to give more subtle responses. 
Such an intuition is further exemplified in the following proposition, the proof of which also echoes the universal existence result in Theorem \ref{thm:unbiased_interpolation}.

\begin{proposition}[Existence of Hierarchical Expectation]
\label{prop:existence_of_hierarchical_expectation}
For any two ordinal feedback systems $Z$ and $Z^\prime$ taking values in $\mathcal{Z}$ and $\mathcal{Z}^\prime$ over the same probability space $(\Omega, \mathcal{F}, \mathbb{P})$, suppose the marginal distribution of $Z$ is of measure $\mu = \sum_{z_i \in \mathcal{Z}} \alpha_i \delta_{z_i}$ and that of $Z^\prime$ is $\mu^\prime = \sum_{z_i^\prime \in \mathcal{Z}^\prime} \alpha_i^\prime \delta_{z_i^\prime}$, where $\delta(\cdot)$ is the Dirac delta distribution. If there exist real numbers $\beta_{j,k} \in [0, 1]$ such that
\begin{enumerate}
    \item $z_j$ is a convex combination of $z_k^\prime$'s with coefficients $\beta_{j,k}$'s. That is, \\
    $\sum_{k, z_k^\prime \in \mathcal{Z}^\prime} \beta_{j, k} = 1$ and $\sum_{k, z_k^\prime \in \mathcal{Z}^\prime} \beta_{j, k} z_k^\prime = z_j$ for any $z_j \in \mathcal{Z}$;
    \item $\sum_{j, z_j \in \mathcal{Z}} \beta_{j, k} \alpha_j = \alpha_k^\prime$ for any $z_k^\prime \in \mathcal{Z}^\prime$;
\end{enumerate}
Then $Z$ must be a hierarchical expectation of $Z^\prime$. On the other hand, if $Z$ is a hierarchical expectation of $Z^\prime$ according to coupling $(W, W^\prime)$ on $(\Omega_0, \mathcal{F}_0, \mathbb{P}_0)$, then there must exist real numbers $\beta_{j,k} \in [0, 1]$ satisfying the above requirements by setting $\beta_{j, k} = \mathbb{P}_0(W^{\prime} = z_k^\prime | W = z_j)$.
\end{proposition}


The following corollary gives some concrete examples (perhaps the most popular ones).

\begin{corollary}
\label{corol:hierarchical_expectation_examples}
Suppose $\mathcal{Z}$ and $\mathcal{Z}^\prime$ are two ordinal feedback sets satisfying Assumption \ref{assm:wisdom_of_crowd}. Then $Z$ is always a hierarchical expectation of $Z^\prime$ if any one of the following two condition holds:
\begin{enumerate}
    \item $Z$ is the oracle model such that $Z = z_{\text{oracle}}$;
    \item $Z^\prime$ is the binary feedback such that $\mathcal{Z}^\prime = \{0, 1\}$.
\end{enumerate}
\end{corollary}

With these above definitions, we are ready to characterize the finite sample properties of different ordinal feedback systems through the lens of Rademacher complexity.


\begin{definition}[Rademacher Complexity]
Let $\mathcal{D}_{\mathcal{Z}} = \{(x_i, y_{i, 1}, y_{i, 2}, Z_i)\}_{i=1}^n$ be a dataset with $Z$ taking values in $\mathcal{Z}$. Consider a hypothesis class $\mathcal{H}$ of real-valued functions over $\mathcal{X} \times \mathcal{Y}^2$ and a loss function $\ell$. Then, the empirical Rademacher complexity is defined as
\[
\mathrm{Rad}_{\mathcal{D}_{\mathcal{Z}}}(\ell \circ \mathcal{H}) \coloneqq \frac{1}{n} \mathbb{E}_{\varepsilon}\left[\sup_{h \in \mathcal{H}} \sum_{i=1}^n \varepsilon_i \ell\big(Z_i, h(x_i, y_{i, 1}, y_{i, 2})\big)\right],
\]
where $\varepsilon_i$'s are independent random variables all taking values in $\{+1, -1\}$ with equal chances. By assuming that $(x_i, y_{i, 1}, y_{i, 2}, Z_i)$'s are i.i.d. and taking expectations over the entire distribution $\mathbb{P}$, we have the Rademacher complexity as
\[
\mathrm{Rad}_{\mathcal{Z}, n}(\ell \circ \mathcal{H}) \coloneqq \mathbb{E}_{\mathcal{D}_{\mathcal{Z}}\sim \mathbb{P}^n}\big[\mathrm{Rad}_{\mathcal{D}_{\mathcal{Z}}}(\ell \circ \mathcal{H})\big].
\]
\end{definition}

The following theorem says that a more fine-grained feedback system leads to a smaller Rademacher complexity for any hypothesis class $\mathcal{H}$, where setting $\mathcal{H} = \{r_\theta | \theta \in \Theta\}$ yields the results in RM learning. We include how the Rademacher complexity consequently affects the generalization bound to Appendix \ref{subapd:gen_bound} for completeness.

\begin{theorem}
\label{thm:Rademacher_complexity}
Suppose the loss function $\ell$ satisfies the affinity to feedback condition \eqref{eqn:loss_affine} and $\mathcal{H}$ is a hypothesis class of real-valued functions over $\mathcal{X} \times \mathcal{Y}^2$. For any two ordinal feedback systems $Z$ and $Z^\prime$ taking values in $\mathcal{Z}$ and $\mathcal{Z}^\prime$ such that $Z$ is a hierarchical expectation of $Z^\prime$, we have
\[
\mathrm{Rad}_{\mathcal{Z}, n}(\ell \circ \mathcal{H}) \leq \mathrm{Rad}_{\mathcal{Z}^\prime, n}(\ell \circ \mathcal{H}).
\]
\end{theorem}

As a consequence, the following corollary states that any feedback system that satisfies Assumption \ref{assm:wisdom_of_crowd} can be viewed as somewhere in the middle of the binary feedback system and the ideal oracle feedback system $z_{\text{oracle}}$. It is impossible to access the population preference model $z_{\text{oracle}}$ through human annotators in practice, but any ordinal feedback system provides more fine-grained information and leads us towards $z_{\text{oracle}}$.

\begin{corollary}[Ordinal feedback always better than binary]
\label{corol:ordinal_better_than_binary}
Suppose $Z$ is an ordinal feedback taking values in $\mathcal{Z}$ and $Z^\prime$ is a binary feedback in $\mathcal{Z}^\prime = \{0, 1\}$ over the same probability space. If they both satisfy Assumption \ref{assm:wisdom_of_crowd} and the loss function $\ell$ satisfies condition \eqref{eqn:loss_affine}, then
\[
\mathrm{Rad}_{\mathcal{Z}_{\text{oracle}}, n}(\ell \circ \mathcal{H}) \leq
\mathrm{Rad}_{\mathcal{Z}, n}(\ell \circ \mathcal{H}) \leq
\mathrm{Rad}_{\mathcal{Z}^\prime, n}(\ell \circ \mathcal{H}).
\]
\end{corollary}

\subsection{Generalizations to DPO}
\label{subsec:extensions_DPO}

Our results naturally extend to the direct policy optimization (DPO) \citep{rafailov2024direct} case. We first give a quick introduction to the DPO training objective. Suppose the ground-truth reward function for any prompt-response pair $(x, y)$ is $r^*(x, y)$. Then under the reinforcement learning objective of maximizing the reward (with a Kullback-Leibler divergence regularization of strength $\beta$ from the original policy $\pi_{\text{ref}}$), the optimal policy under the ground truth reward function $r^*$ should be
\[
\pi^*(y|x) \propto \pi_{\text{ref}}(y|x) \cdot \exp\left(\frac1\beta r^*(x, y)\right).
\]
Under the Bradley-Terry model, we have
\[
\mathbb{P}(y_1 \succ y_2 | x) = \sigma\left(\beta \log \frac{\pi^*(y_1|x)}{\pi_{\text{ref}}(y_1|x)} - \beta \log \frac{\pi^*(y_2|x)}{\pi_{\text{ref}}(y_2|x)}\right),
\]
where the $\sigma(\cdot)$ is the sigmoid function $\sigma(x) = \exp(x)/(1+\exp(x))$. Then the DPO training objective is to minimize the cross-entropy loss under the binary feedback setting
\begin{align}
\min_{\theta} \sum_{i=1}^n & -Z_i \cdot \log\left(\sigma\left(\beta \log \frac{\pi_\theta(y_{i,1}|x_i)}{\pi_{\text{ref}}(y_{i,1}|x_i)} - \beta \log \frac{\pi_\theta(y_{i,2}|x_i)}{\pi_{\text{ref}}(y_{i,2}|x_i)}\right)\right) \nonumber \\
& - (1-Z_i) \cdot \log\left(\sigma\left(\beta \log \frac{\pi_\theta(y_{i,2}|x_i)}{\pi_{\text{ref}}(y_{i,2}|x_i)} - \beta \log \frac{\pi_\theta(y_{i,1}|x_i)}{\pi_{\text{ref}}(y_{i,1}|x_i)}\right)\right).
\label{eqn:learning_objective_DPO}
\end{align}
By considering a richer feedback than the binary case $\mathcal{Z} = \{0, 1\}$, we can use \eqref{eqn:learning_objective_DPO} to train the LLM $\pi_{\theta}$ directly under the ordinal feedback. The affinity condition \eqref{eqn:loss_affine} is fulfilled since the loss is still the cross-entropy loss. Thus, applying Theorem \ref{thm:Rademacher_complexity} for $\Pi = \{\pi_\theta, \theta \in \Theta\}$ yields a similar result. We present here without repeating the proof.
\begin{corollary}
Consider the corresponding loss function $\ell$ in the training objective \eqref{eqn:learning_objective_DPO} and a policy class $\Pi = \{\pi_\theta, \theta \in \Theta\}$. For any two ordinal feedback systems $Z$ and $Z^\prime$ taking values in $\mathcal{Z}$ and $\mathcal{Z}^\prime$ such that $Z$ is a hierarchical expectation of $Z^\prime$, we have
\[
\mathrm{Rad}_{\mathcal{Z}, n}(\ell \circ \Pi) \leq \mathrm{Rad}_{\mathcal{Z}^\prime, n}(\ell \circ \Pi).
\]
\end{corollary}

\subsection{Implications on soft labeling}
\label{subsec:connect_soft_label}

As noted earlier, the results developed above also have implications on the technique of soft labeling, which we elaborate on in this subsection. Specifically, we show how the analysis can be applied to the context soft labeling and induces a novel bias-variance trade-off for knowledge distillation.

For general $k$-nary classification problems, the standard feedback (labeling of the target variable) is a $k$-dimensional one-hot vector. However, these all-zero-but-one labels make the model overfit easily from a training perspective. Knowledge distillation \citep{hinton2015distilling} is a well-known technique first developed in computer vision to regularize the model from fitting the noises. The original data is used to train a teacher model of which the predictions are named \textit{soft labels}. Then a student model is trained to mimic the predictions of the teacher model, that is, minimizing the training loss against the soft labels generated by the teacher model rather than the original ones. Such a distillation regularizes the student model from overfitting and overconfidence.

Existing theoretical works \citep{phuong2019towards, zhou2021rethinking} are developed to understand the benefits of knowledge distillation and soft labeling. Our theoretical perspective in the preceding subsection provides a new perspective on that problem. In a word, the trained teacher model could be viewed as one (possibly biased) oracle feedback, and learning from the oracle eases the overfitting via reducing the labeling variance (and hence a smaller Rademacher complexity). More concretely, consider the following four learning paradigms:
\begin{enumerate}
\item Oracle/ideal: learn from oracle labeled samples (denoted as $(x, y_{\text{oracle}})$'s in this subsection). The oracle $y_{\text{oracle}}$ is the conditional expectation of the label $y$.
\item Original: learn from the original samples (denoted as $(x, y)$'s in this subsection), where we regard $y$ as a randomly sampled label according to $y \sim y_{\text{oracle}}$.
\item Knowledge distillation: learn from the teacher model $\mathcal{T}$'s output (denoted as $(x, \bar{y}_{\mathcal{T}})$'s in this subsection). The labels $\bar{y}_{\mathcal{T}}$'s are random vectors of which the randomness comes from the teacher model $\mathcal{T}$.
\item Sampling from teacher: learn from the teacher model, but not directly from $\bar{y}_{\mathcal{T}}$; instead, we use a sampled label $y_{\mathcal{T}} \sim \bar{y}_{\mathcal{T}}$. This introduces more randomness in the labeling process.
\end{enumerate}

What is the ``bias-variance'' tradeoff in the learning paradigm (c)? First, $y_{\text{oracle}}$ shares the same conditional expectation with $y$, and so do $\bar{y}_{\mathcal{T}}$ and $y_{\mathcal{T}}$, where their population cross-entropy losses are the same due to the affinity condition \eqref{eqn:loss_affine}. The ``bias'' comes from $\bar{y}_{\mathcal{T}}$ as an imperfect estimation of $y_{\text{oracle}}$, leading to different population losses, while training according to (c) or (d) introduces an additional loss in the original population loss. That is the ``bias'' term
\[
\text{Bias} \coloneqq \mathbb{E}_{\mathcal{T}}\Big[\mathbb{E}_{x, y^\prime}\big[\ell(y^\prime, h^*_{\mathcal{T}}(x))\big] - \mathbb{E}_{x, y^\prime}\big[\ell(y^\prime, h^*(x))\big] \Big] \geq 0,
\]
where we set $y^\prime$ to be i.i.d. as $y$ to prevent the dependence of $\mathcal{T}$ on $y$, and
\[
h^*_{\mathcal{T}} \coloneqq \argmin_{h \in \mathcal{H}} \mathbb{E}_{x, \bar{y}_{\mathcal{T}}}\big[\ell(\bar{y}_{\mathcal{T}}, h(x))\big], \quad h^* \coloneqq \argmin_{h \in \mathcal{H}} \mathbb{E}_{x, y}\big[\ell(y, h(x))\big],
\]
are the corresponding hypotheses minimizing the population losses.

As for the variance, we can directly see from the construction that $\bar{y}_{\mathcal{T}}$ (or $y_{\text{oracle}}$) is a hierarchical expectation of $y_{\mathcal{T}}$ (or $y$), thus learning paradigm (c) (or (a)) has a lower Rademacher complexity compared to that of (d) (or (b)):
\[
\mathrm{Rad}_{y_{\mathcal{T}}}(\ell \circ \mathcal{H}) - \mathrm{Rad}_{\bar{y}_{\mathcal{T}}}(\ell \circ \mathcal{H}) \geq 0 \quad \text{for any }\mathcal{T}, \quad
\mathrm{Rad}_{y}(\ell \circ \mathcal{H}) - \mathrm{Rad}_{y_{\text{oracle}}}(\ell \circ \mathcal{H}) \geq 0,
\]
where the subscript of the Rademacher complexity denotes the source distribution of the label. However, such an argument does not directly compare the Rademacher complexity of the original learning paradigm ($\mathrm{Rad}_{y}$) with that of the knowledge distillation ($\mathrm{Rad}_{\bar{y}_{\mathcal{T}}}$). To explicitly show that the ``variance'' is reduced, we make the following assumption that the marginal output of the teacher model is unbiased:
\begin{assumption}[Soft-labeling version of Assumption \ref{assm:wisdom_of_crowd}]
\label{assm:marginal_unbiased_teacher}
We assume that the teacher model is marginally unbiased. That is,
\[
\mathbb{E}_{\mathcal{T}}[\bar{y}_{\mathcal{T}} | x] = y_{\text{oracle}}(x).
\]
Here $\mathcal{T}$ is the teacher model trained from original labels $y$'s. The randomness of $\mathcal{T}$ comes from the randomness of $y$'s and the training procedure (e.g. random seeds).
\end{assumption}
In general, the learned teacher model outputs $\bar{y}_{\mathcal{T}}$'s are biased estimations of $y_{\text{oracle}}$'s, and the biases are towards the original labels $y$'s (in the extreme case where $\mathcal{T}$ interpolates all the labels, we have $\bar{y}_{\mathcal{T}} = y$). We note that the assumption only requires an unbiasedness in a marginal sense, that those biases cancel out each other in the marginal sense. We have the following with the assumption (as an analogy of Assumption \ref{assm:wisdom_of_crowd}).
\begin{theorem}
\label{thm:soft_label_benefits}
Under Assumption \ref{assm:marginal_unbiased_teacher} and the cross-entropy loss, we have
\[
\text{Reduced variance} \coloneqq \mathrm{Rad}_{y}(\ell \circ \mathcal{H}) - \mathrm{Rad}_{\bar{y}_{\mathcal{T}}}(\ell \circ \mathcal{H}) \geq 0.
\]
\end{theorem}

In this light, the notion of hierarchical expectation and the reduced Rademacher complexity in Theorem \ref{thm:soft_label_benefits} render a new bias-variance tradeoff for the knowledge distillation methods. Compared to that of \citet{zhou2021rethinking}, our approach theoretically shows the variance is always reduced by introducing the soft labels, while \citet{zhou2021rethinking} makes the reduction an assumption and verifies it empirically.

\section{Numerical Experiments}

\label{sec:experiments}

We perform numerical experiments to answer two questions: (i) How do different granularities of the feedback model affect the learning of the reward model? (ii) Does the inclusion of these ordinal feedback data with the objective \eqref{eqn:learning_objective} benefit the learning of the reward model?


\subsection{Experiment Settings}

\label{subsec:experiment_setup}
\textbf{Datasets.} In the following numerical experiments, we leverage the Skywork-Reward-Preference-80K-v0.2 dataset \citep{liu2024skywork} as our base training dataset. We perform multiple runs and report the average performance (along with the confidence intervals); for each run, we randomly sample a 1024-sized subset as the hold-out evaluation dataset. In addition, we use the RewardBench dataset \citep{lambert2024rewardbench} for the out-of-distribution evaluation task to comprehensively assess the performance of different trained models. More details can be found in Appendix \ref{sec:dataset_details}.

\textbf{Base Models.} Our base models for the following experiments are llama-3.2-1b-instruct \citep{dubey2024llama} and gemma-2-2b-it \citep{team2024gemma}. Both models are trained under full-parameter fine-tuning. For the training parameters and other details, we refer to Appendix \ref{sec:training_details}.

\textbf{Ordinal Feedback.} The original Skywork-Reward-Preference-80K-v0.2 dataset only contains a binary feedback for each prompt $x$ and a response pair $y_1$ and $y_2$. To generate feedback labels with different levels of granularities, we adopt a well-trained reward model, Skywork-Reward-Gemma-2-27B-v0.2 \citep{liu2024skywork}, as the oracle scoring model $r_{\text{oracle}}: \mathcal{X} \times \mathcal{Y} \rightarrow \mathbb{R}$ in this case. We chose this model because (1) it was exclusively trained on the to-be-labeled base training dataset hence there is hardly a risk of out-of-distribution mislabeling; (2) the model ranks first on the RewardBench online leaderboard up to the time of this paper, making its output oracle scores more reliable. Accordingly, the induced oracle model being $z_{\text{oracle}}(x, y_1, y_2) = \sigma((r_{\text{oracle}}(x, y_1) - r_{\text{oracle}}(x, y_2))/T)$ where $T$ is a temperature parameter and $\sigma(\cdot)$ denotes the sigmoid function. 

We consider the following four types of feedback systems: 
\begin{itemize}
    \item Oracle: $z_{\text{oracle}}$ is directly used as the feedback label and $\mathcal{Z}_{\text{oracle}} \subset [0, 1]$.
    \item Binary: the label is sampled by $Z_{\text{binary}} \sim \mathrm{Bernoulli}(z_{\text{oracle}})$ and $\mathcal{Z}_{2} = \{0, 1\}$.
    \item 3-level: the label is sampled as the process in Theorem \ref{thm:unbiased_interpolation} considering only the smallest interval containing $z_{\text{oracle}}$ and $\mathcal{Z}_{3} = \{0, 0.5, 1\}$.
    \item 5-level: the label is sampled as the process in Theorem \ref{thm:unbiased_interpolation} considering only the smallest interval containing $z_{\text{oracle}}$ and $\mathcal{Z}_{5} = \{0, 0.2, 0.5, 0.8, 1\}$.
\end{itemize}

We provide a label histogram in Appendix \ref{sec:dataset_details}. We adopt the objective function \eqref{eqn:learning_objective} to train the reward model.

\subsection{Fine-grained feedback leads to better reward learning}
\label{subsec:HE_experiment}

As discussed earlier, a more fine-grained feedback system should intuitively and theoretically lead to better reward learning. For the four feedback models listed above, they should have the following orders in terms of performance:
$$\text{Oracle} \ge \text{5-level} \ge \text{3-level}  \ge \text{Binary}$$
where $\ge$ represents an advantage in model performance.

Here we perform numerical experiments to verify such intuitions and for each combination of the reward model and feedback system, we conduct 5 independent training runs and report the average results.  For the setting of learning with oracle feedback, we set 
$$Z_i=\mathbb{P}\left(y_{i,1} \succ y_{i,2} | x_i\right) = z_{\text{oracle}}(x_i, y_{i,1}, y_{i,2})$$
in the learning objective \eqref{eqn:learning_objective}. For more fine-grained ordinal feedback, we sample the feedback according to Section \ref{subsec:experiment_setup}.

Table \ref{tab:HE_experiment_results} and Figure \ref{fig:HE_experiment_results} summarize the experiment results which are aligned with the findings in the previous sections. Three take-away messages are: First, a more fine-grained feedback structure leads to better reward learning for both in-distribution (ID) and out-of-distribution (OOD) performance. Second, though we do not have access to the oracle model in practice, the 5-level feedback system provides a good proxy for that. Third, the learning objective \eqref{eqn:learning_objective}, as a generalization of the canonical cross-entropy loss for binary feedback, is an effective one to handle the ordinal feedback data. 


\begin{table}[ht!]
\centering
\begin{tabular}{cccccccc}
\toprule
\multirow{2}{*}{\text{Model}} & \multirow{2}{*}{\text{Feedback}}& \multicolumn{2}{c}{\text{Oracle CE Loss}} & \multicolumn{2}{c}{\text{ID Accuracy}} & \multicolumn{2}{c}{\text{OOD Accuracy}} \\
\cmidrule(r){3-4} \cmidrule(r){5-6} \cmidrule(r){7-8}
 && Mean & Std & Mean & Std & Mean & Std \\
\midrule
\multirow{4}{*}{\text{Llama}}&Oracle & 0.5711 & 0.0020 & 0.9382 & 0.0037 & 0.8193 & 0.0016 \\
&5-level & 0.5714 & 0.0019 & 0.9372 & 0.0040 & 0.8100 & 0.0013 \\
&3-level & 0.5715 & 0.0021 & 0.9359 & 0.0044 & 0.8016 & 0.0034 \\
&Binary & 0.5736 & 0.0024 & 0.9329 & 0.0044 & 0.7667 & 0.0012 \\
\midrule
\multirow{4}{*}{\text{Gemma}}&Oracle & 0.5698 & 0.0018 & 0.9401 & 0.0031 & 0.8697 & 0.0072 \\
&5-level & 0.5704 & 0.0016 & 0.9371 & 0.0082 & 0.8584 & 0.0107 \\
&3-level & 0.5704 & 0.0018 & 0.9381 & 0.0083 & 0.8580 & 0.0016 \\
&Binary & 0.5709 & 0.0021 & 0.9368 & 0.0074 & 0.8237 & 0.0101 \\
\bottomrule
\end{tabular}
\caption{\small Model convergence statistics under different feedback models. ID stands for in-distribution. OOD stands for out-of-distribution. The ID and OOD datasets are in Section \ref{subsec:experiment_setup}. The oracle CE loss is computed by adopting $z_{\text{oracle}}$ as $Z_i$ regardless of the feedback type.} 
\label{tab:HE_experiment_results}
\end{table}

\begin{figure}[ht!]
\centering
\begin{subfigure}{0.48\textwidth}
    \centering
    \includegraphics[width=\linewidth]{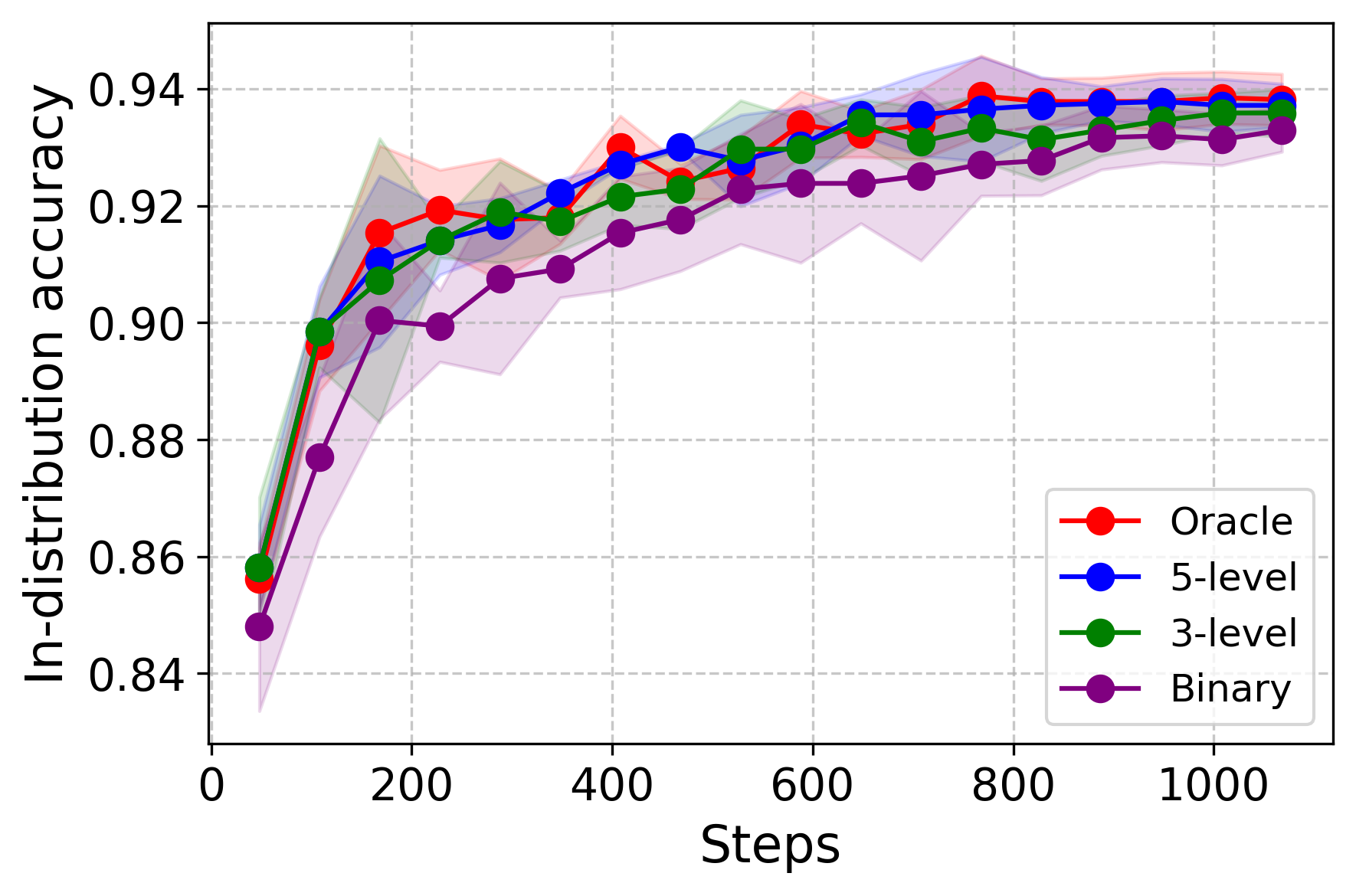}
    \caption{In-distribution (ID) accuracy/llama models}
    \label{fig:HE_experiment_results_llama_acc_finer_feedback_better}
\end{subfigure}
\hfill
\begin{subfigure}{0.48\textwidth}
    \centering
    \includegraphics[width=\linewidth]{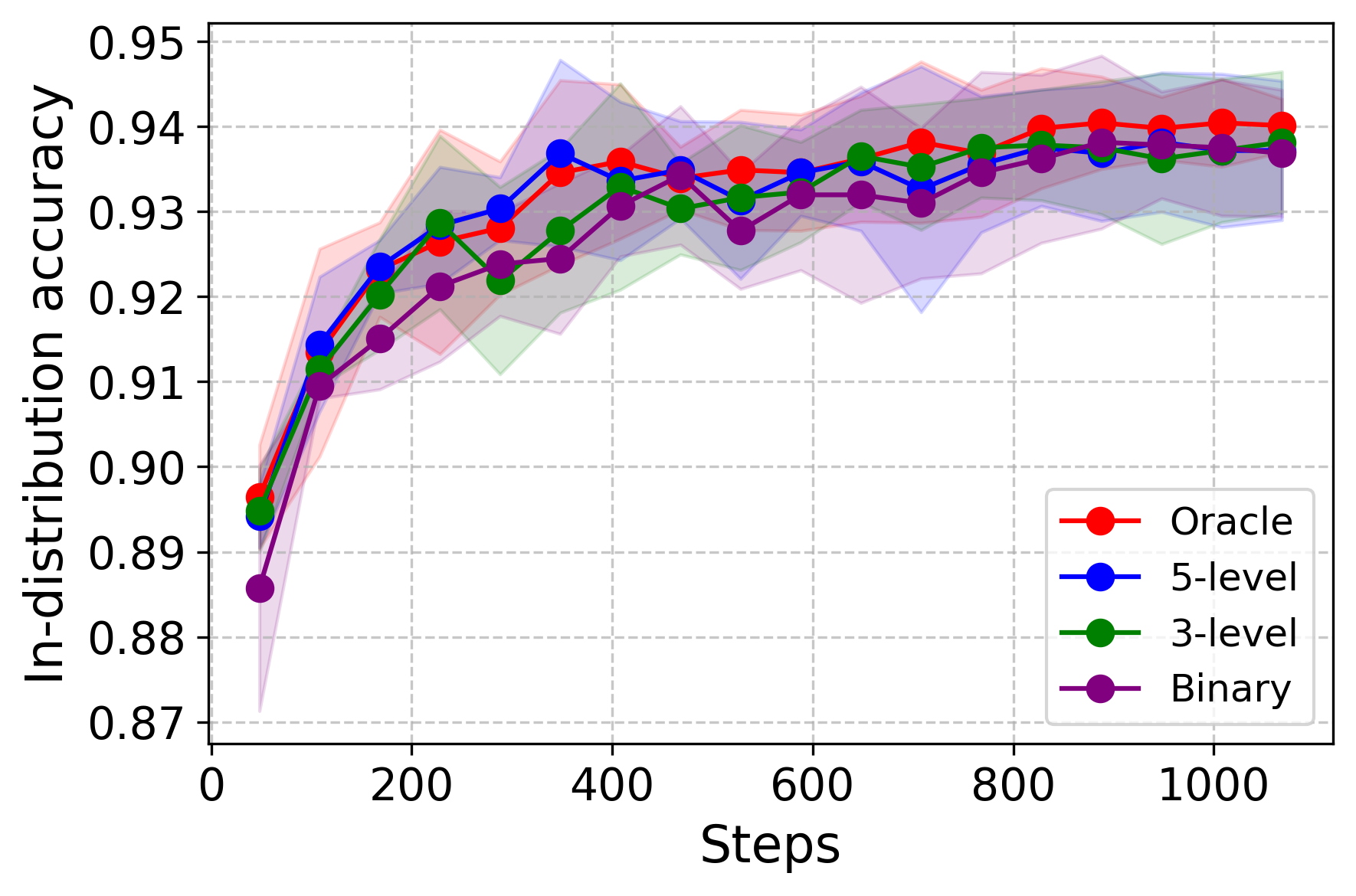}
    \caption{In-distribution (ID) accuracy/gemma models}
    \label{fig:HE_experiment_results_gemma_acc_finer_feedback_better}
\end{subfigure}

\vspace{0.5cm} 

\begin{subfigure}{0.48\textwidth}
    \centering
    \includegraphics[width=\linewidth]{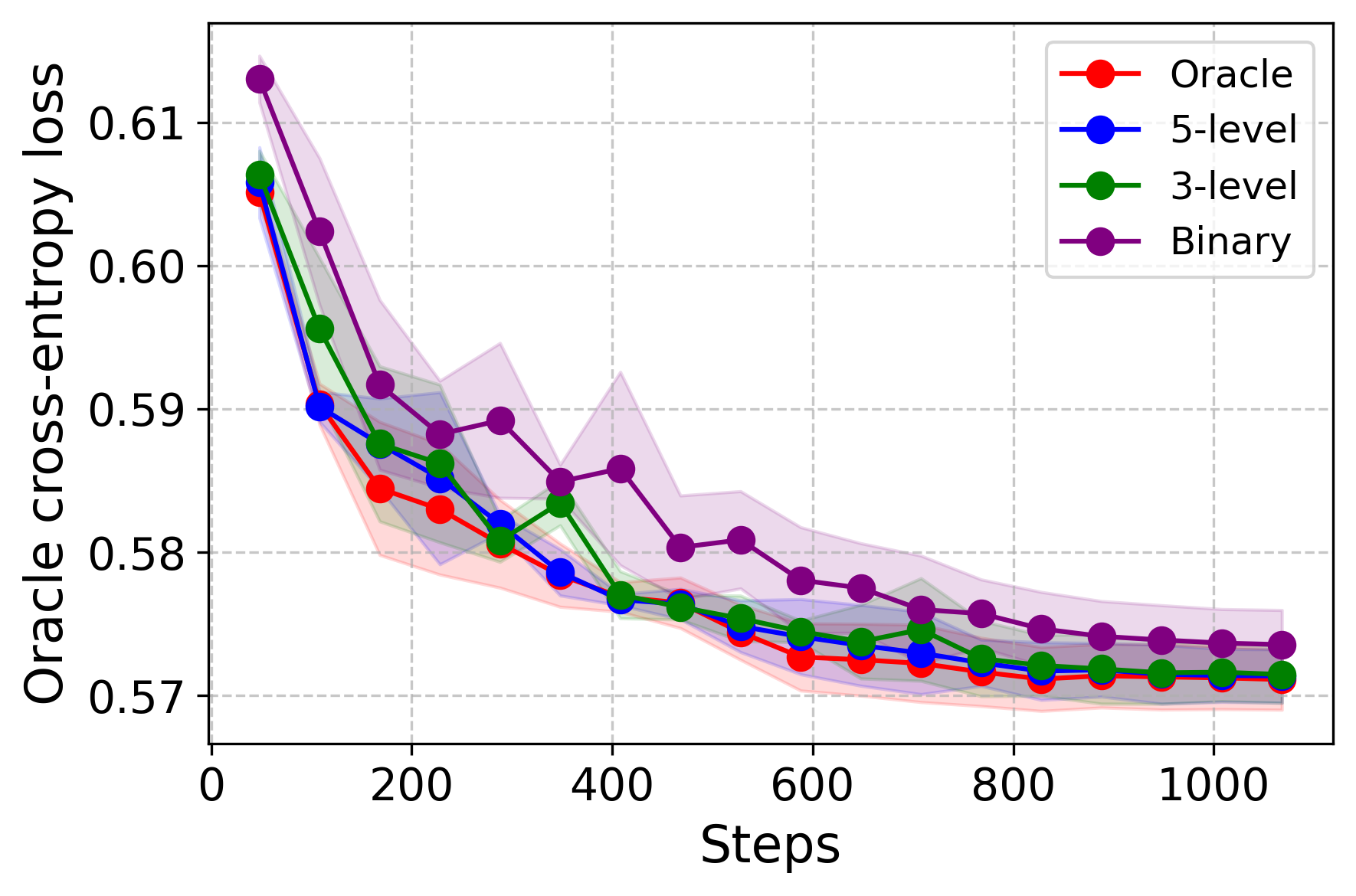}
    \caption{Oracle CE loss/llama models}
    \label{fig:HE_experiment_results_llama_loss_finer_feedback_better}
\end{subfigure}
\hfill
\begin{subfigure}{0.48\textwidth}
    \centering
    \includegraphics[width=\linewidth]{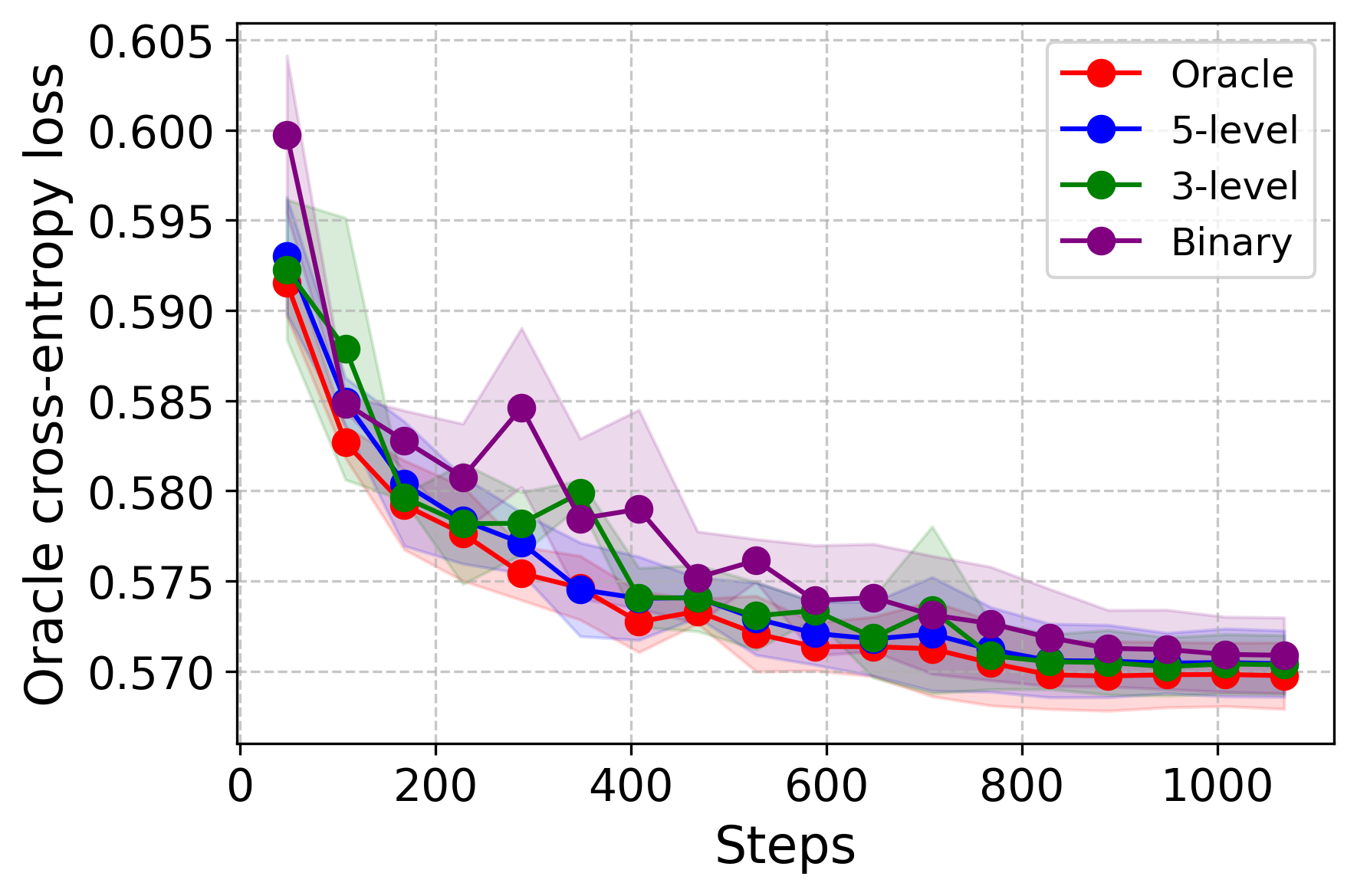}
    \caption{Oracle CE loss/gemma models}
    \label{fig:HE_experiment_results_gemma_loss_finer_feedback_better}
\end{subfigure}

\caption{The evaluation dynamics of llama and gemma models for different ordinal feedback labels.}
\label{fig:HE_experiment_results}
\end{figure}

\subsection{Ordinal feedback v.s. binary feedback}
\label{subsec:POS_experiment}

Now we restrict our attention to the 3-level feedback setting and investigate the effect of the proportion of the tied data (samples with labels of $y_1$ ``same as'' $y_2$). Specifically, we limit the training samples to 32,768 and consider 5 different proportions of the tied data:
\begin{itemize}
\item 0\%-tied: All the data samples are binary-labeled. 
\item 25\%, 50\%, 75\% of the data samples are tied.
\item 100\%-tied: All the data samples are tied. 
\end{itemize}
More details of how the tied data samples are generated are deferred to \ref{sec:dataset_details}. 

Table \ref{tab:pos_experiment_results} and Figure \ref{fig:pos_experiment_results} summarize the experiment results. We make the following observations. First, the 100\%-tied setting fails in that it results in a significantly worse performance than the other settings. This is natural as it leads to a reward collapse, as also observed in other semi-supervised learning algorithms; to see this, if we are given only the tied data, one way to learn the reward model is to have all the rewards equal to a constant. Second, mixing a proportion of the tied data and using the learning objective function \eqref{eqn:learning_objective} leads to a better performance than the case of 0\%-tied data. One subtle point here is that, in practice if we do not employ the learning objective \eqref{eqn:learning_objective} and simply drop the tied samples, this will result in a smaller sample size for learning the reward model, and an even worse performance than the 0\%-tied setting here. Third, if we look into the training dynamics of Figure \ref{fig:pos_experiment_results}, we can see that the curves with tied data samples are smoother than the ones with 0\%-tied samples. This means the inclusion of the tied samples also leads to a smoother loss landscape. 


Recent works \citep{chen2024extending, liu2024reward} have noticed the importance of incorporating tied samples and employed the Rao-Kupper model \citep{rao1967ties}, or the Bradley-Terry model with Ties (BTT) abbreviated by \citet{liu2024reward}, for preference modeling and to explore the benefits of leveraging ties. The BTT model represents preference probabilities as follows: 
\[
\mathbb{P}(y_1 \succ y_2 \mid x) \approx \frac{\exp\big(r_{\theta}(x, y_1)\big)}{\exp\big(r_{\theta}(x, y_1)\big) + \lambda \exp\big(r_{\theta}(x, y_2)\big)},
\]
\[
\mathbb{P}(y_1 \sim y_2 \mid x) \approx \frac{(\lambda^2 - 1)\exp\big(r_{\theta}(x, y_1)\big)\exp\big(r_{\theta}(x, y_2)\big)}{\left(\exp\big(r_{\theta}(x, y_1)\big) + \lambda \exp\big(r_{\theta}(x, y_2)\big)\right)\left(\lambda\exp\big(r_{\theta}(x, y_1)\big) + \exp\big(r_{\theta}(x, y_2)\big)\right)},
\]
where $y_1 \sim y_2$ denotes a tie between $y_1$ and $y_2$ and $\lambda \geq 1$ is a threshold hyperparameter that controls the likelihood of assigning ties. Notably, when $\lambda = 1$, the BTT model reduces to the canonical BT model.

While the BTT model appears promising for accommodating ties, it is not well-suited for preference learning. First, the BTT model introduces an additional hyperparameter $\lambda$ that needs to be tuned, while involving such a more complicated model is unnecessary in reward model learning. The BTT model exactly outputs a term to predict the tied probability, which is useful when predicting football games where the ties have real-world implications. However, the ultimate goal of the RM training is to provide a reference feedback function for the following RLHF step to align the LLM with human preferences: the learned RM need not exactly predict the probability of a tie but only capture the overall trending of human preferences. The complicated BTT model is unnecessary for RM learning. Second, the BTT model is only designed for the 3-level feedback case, which is not enough for the richer feedback system that has already been adopted by those LLM companies. Generalizing the BTT model to the 5-level case is much more sophisticated. There should be other 2 hyperparameters like $\lambda$ to build the 5-level model. Tuning those 3 hyperparameters together may be very difficult and time-consuming. As a comparison, our ordinal feedback generalizes naturally to all levels without introducing any additional hyperparameters. Our work is even not limited to the BT model and can be naturally applied to other probability models of preferences.



\begin{table}[ht!]
\centering
\begin{tabular}{cccccccc}
\toprule
\multirow{2}{*}{\text{Model}} & \multirow{2}{*}{\text{Tied Ratio}}& \multicolumn{2}{c}{\text{Oracle CE Loss}} & \multicolumn{2}{c}{\text{ID Accuracy}} & \multicolumn{2}{c}{\text{OOD Accuracy}} \\
\cmidrule(r){3-4} \cmidrule(r){5-6} \cmidrule(r){7-8}
 && Mean & Std & Mean & Std & Mean & Std \\
\midrule
\multirow{4}{*}{\text{Llama}}&0\% & 1.0421 & 0.0363 & 0.9224 & 0.0080 & 0.7661 & 0.0182 \\
&25\% & 0.3327 & 0.0051 & 0.9341 & 0.0173 & 0.7672 & 0.0093 \\
&50\% & 0.4187 & 0.0043 & 0.9336 & 0.0014 & 0.7545 & 0.0082 \\
&75\% & 0.5339 & 0.0052 & 0.9268 & 0.0180 & 0.7749 & 0.0008 \\
&100\% & 0.6931 & 0.0017 & 0.3428 & 0.0677 & 0.4424 & 0.0393 \\
\midrule
\multirow{4}{*}{\text{Gemma}}&0\% & 6.4762 & 0.3392 & 0.9355 & 0.0041 & 0.8319 & 0.0080 \\
&25\% & 0.6031 & 0.0019 & 0.9467 & 0.0117 & 0.8487 & 0.0100 \\
&50\% & 0.5775 & 0.0001 & 0.9526 & 0.0075 & 0.8277 & 0.0006 \\
&75\% & 0.6122 & 0.0006 & 0.9521 & 0.0069 & 0.8236 & 0.0084 \\
&100\% & 0.6931 & 0.0001 & 0.4814 & 0.0055 & 0.4928 & 0.0158 \\
\bottomrule
\end{tabular}
\caption{Model convergence statistics under different tied data ratios. The evaluation dataset remains fixed across all ratio settings and is directly sampled from the original dataset, ensuring its distribution closely matches that of the whole dataset.}
\label{tab:pos_experiment_results}
\end{table}



\begin{figure}[ht!]
\centering
\begin{subfigure}{0.48\textwidth}
    \centering
    \includegraphics[width=\linewidth]{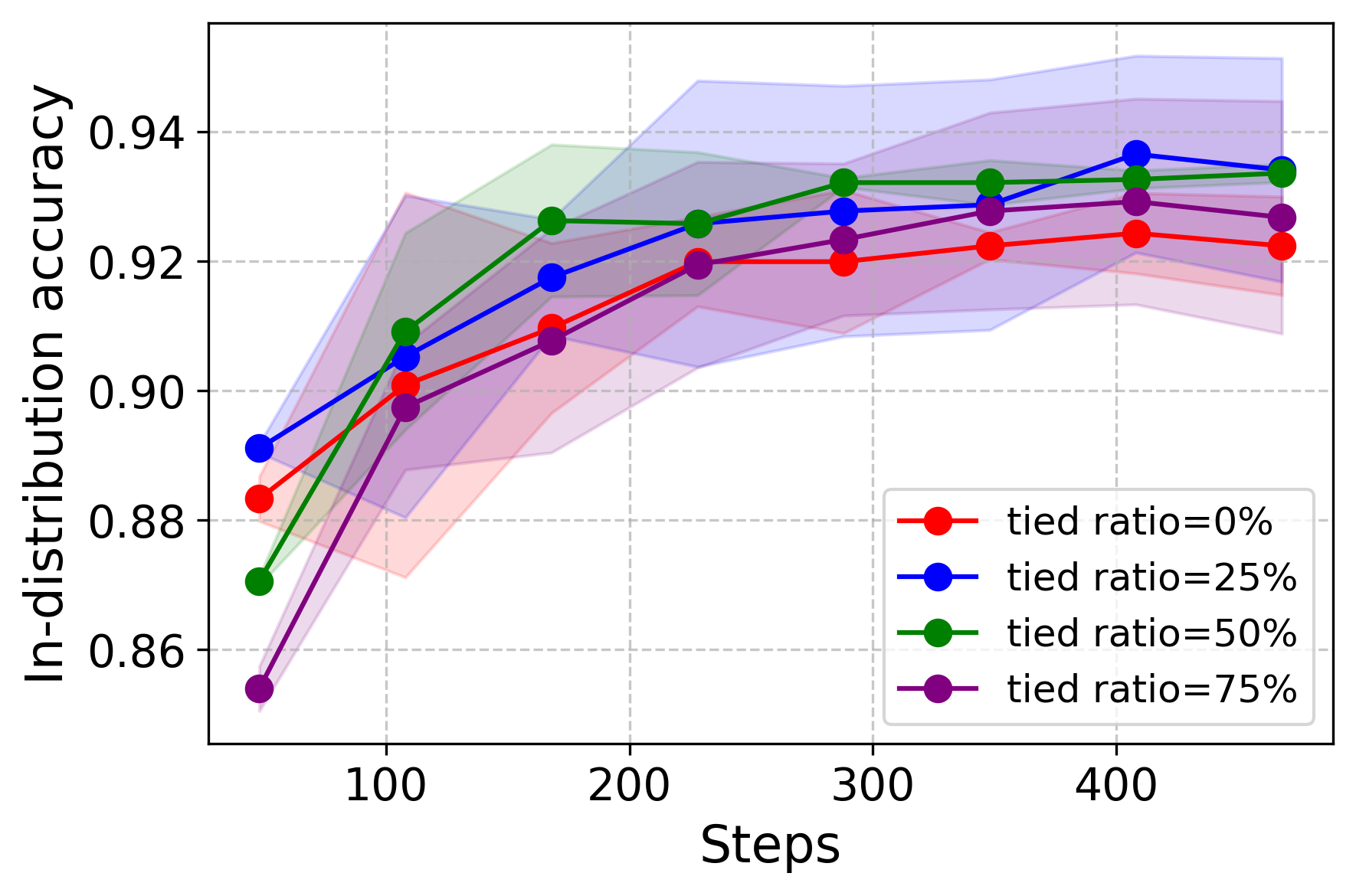}
    \caption{ID accuracy/llama models}
    \label{fig:HE_experiment_results_llama_acc_mixed_ratio_better}
\end{subfigure}
\hfill
\begin{subfigure}{0.48\textwidth}
    \centering
    \includegraphics[width=\linewidth]{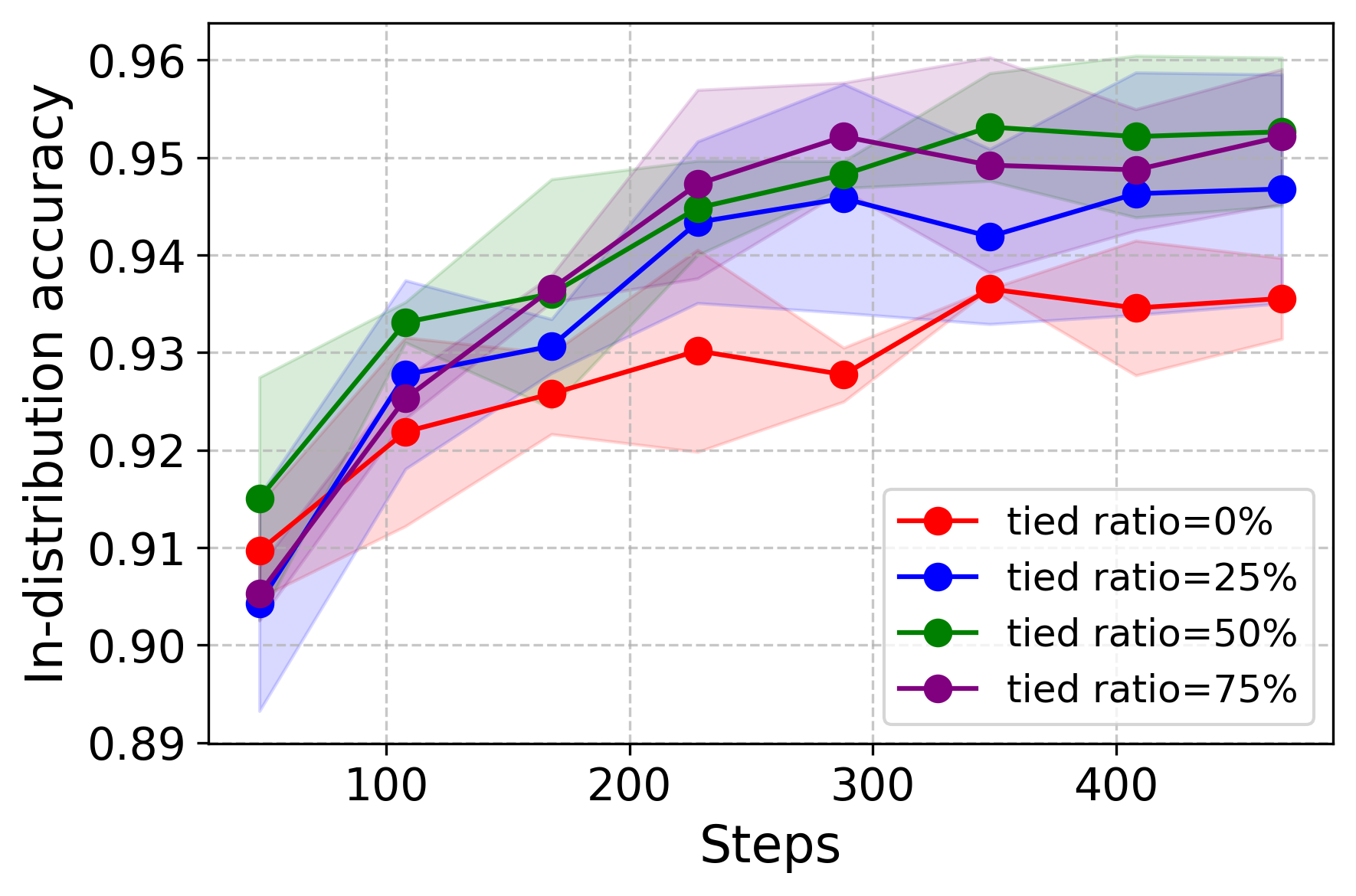}
    \caption{ID accuracy/gemma models}
    \label{fig:HE_experiment_results_gemma_acc_mixed_ratio_better}
\end{subfigure}

\vspace{0.5cm} 

\begin{subfigure}{0.48\textwidth}
    \centering
    \includegraphics[width=\linewidth]{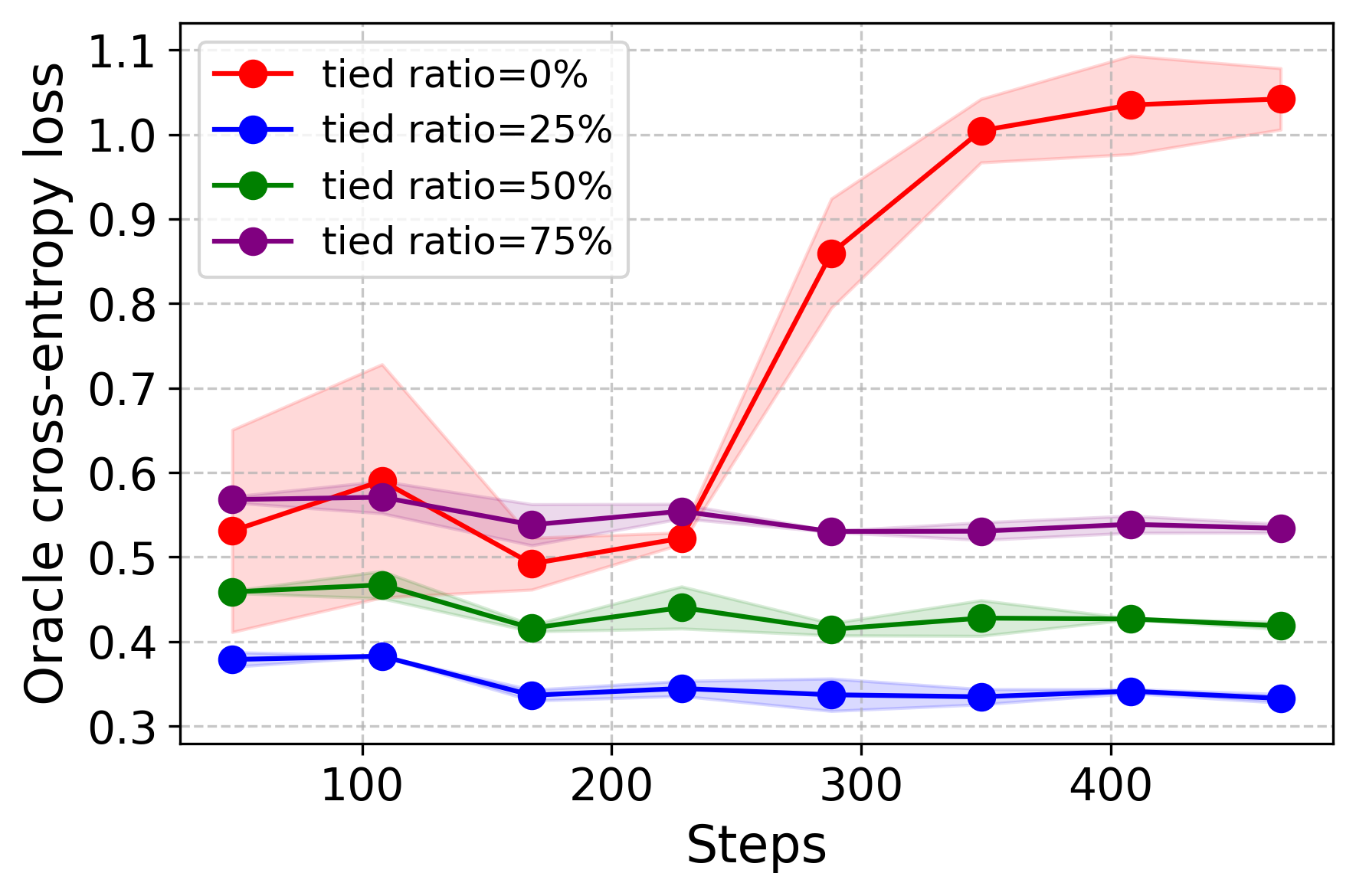}
    \caption{Oracle CE loss/llama models}
    \label{fig:HE_experiment_results_llama_loss_mixed_ratio_better}
\end{subfigure}
\hfill
\begin{subfigure}{0.48\textwidth}
    \centering
    \includegraphics[width=\linewidth]{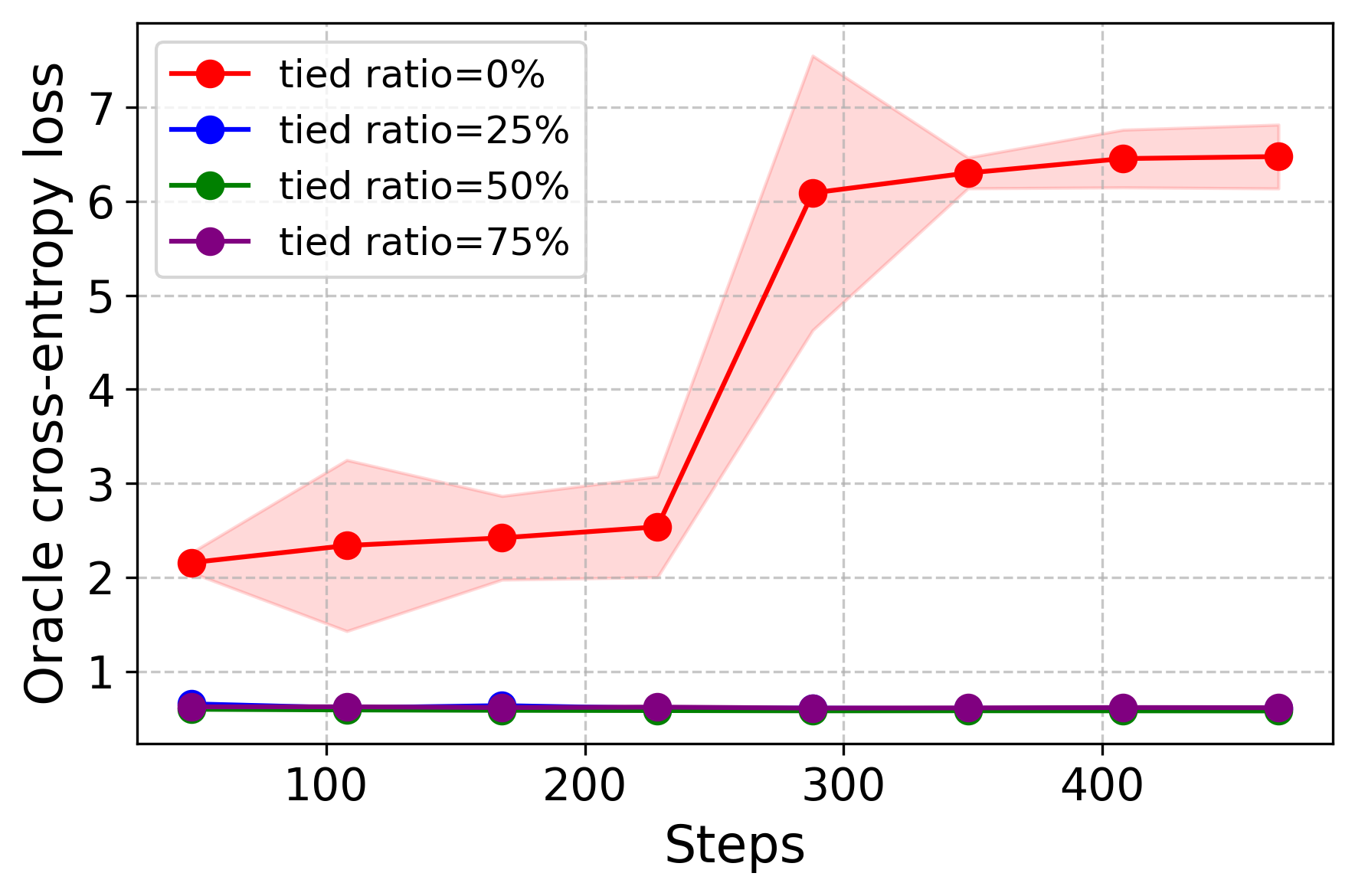}
    \caption{Oracle CE loss/gemma models}
    \label{fig:HE_experiment_results_gemma_loss_mixed_ratio_better}
\end{subfigure}

\caption{The evaluation dynamics of llama and gemma models for different tied data ratios. The 100\%-tied case is not plotted as it would detract from the clarity and readability of the plot due to its failure.}
\label{fig:pos_experiment_results}
\end{figure}


\section{Related Works}
\label{sec:related_works}

Reinforcement learning from human feedback (RLHF) originates from the idea of preference-based reinforcement learning \citep{cheng2011preference, akrour2011preference}. The term RLHF is proposed by the large language model (LLM) community and has been a mainstream framework for aligning LLMs with human preferences \citep{askell2021general, ouyang2022training}. For a more detailed survey of the history of RLHF, we refer to \citet{kaufmann2023survey}. While many people have incorporated the supervised fine-tuning (SFT) stage as a part of the RLHF pipeline \citep{ziegler2019fine, ouyang2022training, ji2023ai}, our discussion would focus on the reward modeling and policy optimization via reinforcement learning (RL) which take a supervised fine-tuned model as the starting point.

Although proven effective for aligning LLMs with human preferences, the canonical RLHF suffers in several aspects, including complicated implementation, difficult hyperparameter tuning, low sampling efficiency \citep{choshen2019weaknesses} and computational overhead \citep{yuan2024rrhf}, which promotes the studies to optimize with the relative preferences without depending on RL. One major alternative to the RLHF is the direct policy optimization (DPO) proposed by \citet{rafailov2024direct}, which directly trains the language model to increase the probability of the preferred response and decrease the other. 
\citet{wang2023beyond} discuss the influence of $f$-divergence as the constraint term in PPO and proposes $f$-DPO.
\citet{azar2024general} consider a general learning objective under pair-wise preferences, i.e., $\Psi$PO, points out the potential overfitting issue in RLHF and PPO, and further mitigates the problem with a specific instance of $\Psi$PO, i.e., IPO.
\citet{zeng2024token} investigate the optimization at a finer level and employs forward KL divergence constraints for each token to improve the alignment.
There are also some methods \citep{xu2024contrastive, ethayarajh2024kto} that try to skip the SFT stage to lower the costs and mitigate the issues while directly imitating the reference data. 
\citet{liu2023statistical} use statistical rejection sampling to address the mismatch between the training data and optimal policy data, hence enhancing the preference data collection.
\citet{amini2024direct} propose DPO with an offset (ODPO) that takes the preference strength into consideration, which forces the language model to separate the probabilities of two responses in the training dataset by an offset. Compared to our work, the authors do not theoretically prove the benefits of considering the reference strength, and their proposal requires tuning the offset.

Parallel to previously mentioned efforts to overcome the dependence on RL, researchers also directly optimize the policy under different loss functions with different types of preference data. We list some efforts non-exhaustively here, including RAFT \citep{dong2023raft}, SLiC \citep{zhao2023slic}, LiPO \citep{liu2024lipo}, RRHF \citep{yuan2024rrhf}, and PRO \citep{song2024preference}. 
For pairwise comparison data, \citet{zhao2023slic} propose a hinge type loss to encourage the LLM to output the chosen sequence more likely than the rejected one. 
For those list data that compare many different responses, \citet{dong2023raft} select the best response and fine-tunes the LLM on those best-of-$K$ data. 
\citet{yuan2024rrhf} adopt the ranking loss in the same spirit of increasing the preferred sequence's likelihood and summing up all the pairwise comparisons. 
\citet{song2024preference} replace the pairwise comparison with the preferred one against the remaining responses and recursively iterate from the most preferred one to the second-least preferred. All those ranking-based methods are purely based on the relative positions without considering the preference strength, and forcing a rank among nearly tied responses introduces additional noise.
To mitigate this issue, \citet{liu2024lipo} apply the LambdaLoss \citep{burges2006learning} to build a LiPO-$\lambda$ objective that takes the preference strength into the weights of the ranking loss. However, the proposed method has two drawbacks. First, the loss is heuristically defined and lacks theoretical guarantees. Second, the requirement of labeling each response with a quantitative reward is not easy for human annotators: on the contrary, the initial motivation behind the RLHF method is to train a reward model by pure \textit{comparison} data to avoid asking human annotators to quantify a \textit{reward} precisely. In contrast, our ordinal feedback only requires the human annotators to calibrate the qualitative \textit{comparison} with a quantity.

With the rapid development of optimization frameworks, some scholars notice the issue of diverse preferences in reward modeling, not only in the field of LLM \citep{dong2023aligndiff}. 
\citet{zeng2024diversified} propose a multi-objective reward learning method (MORE) to calibrate the reward models with the shared preferences and enhance alignment performance. 
\citet{chakraborty2024maxmin} introduce a MaxMin alignment objective to learn a mixture of diverse preference mixtures and greatly improve the overall performance. 
\citet{wang2024arithmetic} employ the multi-objective reward modeling and models the preferences of users, implementing a better objective control. 
\citet{wang2024interpretable} enhance the interpretability and performance of reward models by combining multi-objective reward modeling and mixture-of-experts techniques. Different from these works, our method considers a more practical scenario without requiring multiple labels for preference pairs and constructs the probability model for the single objective preference learning.

Two recent works \citep{chen2024extending, liu2024reward} incorporate the tied samples into the learning of the reward model by considering some generalized versions \citep{rao1967ties, davidson1970extending} of the Bradley-Terry model \citep{bradley1952rank}. By introducing a more complicated model, they enable the reward model to predict three probabilities: better, worse, and tied. The authors directly use the cross-entropy loss for the trinary classification for the RM learning. As a comparison, our ordinal feedback is more general and not limited to the Bradley-Terry model and the 3-level feedback. We also avoid introducing additional hyperparameters and keep the original training paradigm.

\section{Conclusion}
\label{sec:conclusion}
In this paper, we propose reward modeling with ordinal feedback as a generalization of the binary feedback. Such a framework fully uses the potentially useful samples and the fine-grained information discarded by the binary feedback practice. We generalize the assumption of the BT model to the general marginal unbiasedness assumption, which we name by a sociological concept ``wisdom of the crowd''. Under that assumption, we build a natural probability model for ordinal feedback. We also show that the Rademacher complexity is reduced by adopting ordinal feedback. The results also cover other loss functions (for example, the hinge loss) and other paradigms (for example, DPO). Numerical results validate the theoretical findings. Further experiments imply that mixing some tied preference samples benefits RM learning, which may be worth future exploration. Our results suggest that the annotation guideline should encourage the quantitative description (for example, 70\%) of the qualitative option (for example, ``slightly better''). Our theoretical analysis based on hierarchical expectation may be of independent interest to the field of knowledge distillation, providing a novel bias-variance trade-off perspective.

\bibliographystyle{informs2014}
\bibliography{main}

\newpage
\appendix

\section{Reward Modeling with Ordinal Feedback under Hinge Loss}
\label{apd:hinge_experiments}

In the main paper, we focus on reward modeling with ordinal feedback under the cross-entropy loss. Here we extend the analysis to the case of hinge loss. Recall that the learning objective \eqref{eqn:learning_objective} in the main paper can be interpreted as an induction of the cross-entropy loss based on the Bradley-Terry model. Given the probabilistic model of ordinal feedback, the learning objective can naturally be extended to incorporate other types of loss functions. Among these, hinge loss \citep{scholkopf2004kernel} is one of the most widely used, particularly in classification tasks, alongside cross-entropy loss. Hinge loss is commonly associated with Support Vector Machines (SVM) and is characterized by its core principle of enforcing a margin between distinct classes. Building on this observation, we define the learning objective under hinge loss as follows:
\begin{equation}
\min_{\theta} \ 
\sum_{i=1}^n 
Z_{i}\cdot 
\left[ 
	\max\left(
		0, C-\left(r_{\theta}(x_i, y_{i,1}) - r_{\theta}(x_i, y_{i,2})\right)
	\right) 
\right] 
+ 
(1-Z_{i})\cdot 
\left[ 
	\max\left(
		0, C-\left(r_{\theta}(x_i, y_{i,2}) - r_{\theta}(x_i, y_{i,1})\right)
	\right) 
\right],
\label{eqn:hinge_objective}
\end{equation}
where $C$ is the margin hyperparameter that controls the separation between preference classes. When the feedback is binary, i.e., $Z_i \in \{0, 1\}$, the above objective simplifies to the hinge loss commonly used in reward modeling \citep{liu2024skywork}. In our experiments, the margin parameter is tuned with grid search, where the search space is $\{0.5, 1, 2, 4\}$, and we select $C=2$.

To further validate the conclusions presented in Section \ref{sec:experiments}, we conduct analogous experiments under the learning objective \eqref{eqn:hinge_objective}. The feedback in the oracle feedback setting is still kept as
$$Z_i=\mathbb{P}\left(y_{i,1} \succ y_{i,2} | x_i\right) = z_{\text{oracle}}(x_i, y_{i,1}, y_{i,2})$$
just as the cross-entropy setting. The feedback in the 5-level/3-level/binary settings is sampled as the process in Theorem \ref{thm:unbiased_interpolation} by considering only the smallest interval $[z_j, z_{j+1}] \ni z_{\text{oracle}}$. We then replicate the experimental setup from Section \ref{subsec:HE_experiment}, employing llama-3.2-1b-instruct as our base model. The experiment results are reported in Figure \ref{fig:GH_experiment_results} and Table \ref{tab:GH_experiment_results}.

Experiments illustrate the same three conclusions in Section \ref{subsec:HE_experiment}: (i) more fine-grained feedback structures result in better learning for both ID and OOD performance; (ii) the fine-grained feedback (e.g. 5-level) may be a good proxy for the oracle; (iii) the generalized hinge loss handle the feedback richer than binary well.

Apart from the observations analogous to Section \ref{subsec:HE_experiment}, we observe that the hinge objective performs weaker than the cross-entropy objective. We attribute this outcome to two key factors. First, the margin hyperparameter $C$ significantly impacts the model's convergence speed and overall performance, which we only tune with a coarse grid due to computational constraints. Second, the inherent nature of the hinge objective means that only a specific subset of data points influences the final decision boundary. In contrast, the cross-entropy objective leverages the entire dataset during optimization. Given the complexity of language modeling and the intricacies of semantic space, we hypothesize that preference data embeddings are distributed in a noisy and overlapping manner such that the decision boundary may not be effectively established under the hinge objective.

\begin{figure}[ht!]
\centering
\begin{subfigure}{0.48\textwidth}
    \centering
    \includegraphics[width=\linewidth]{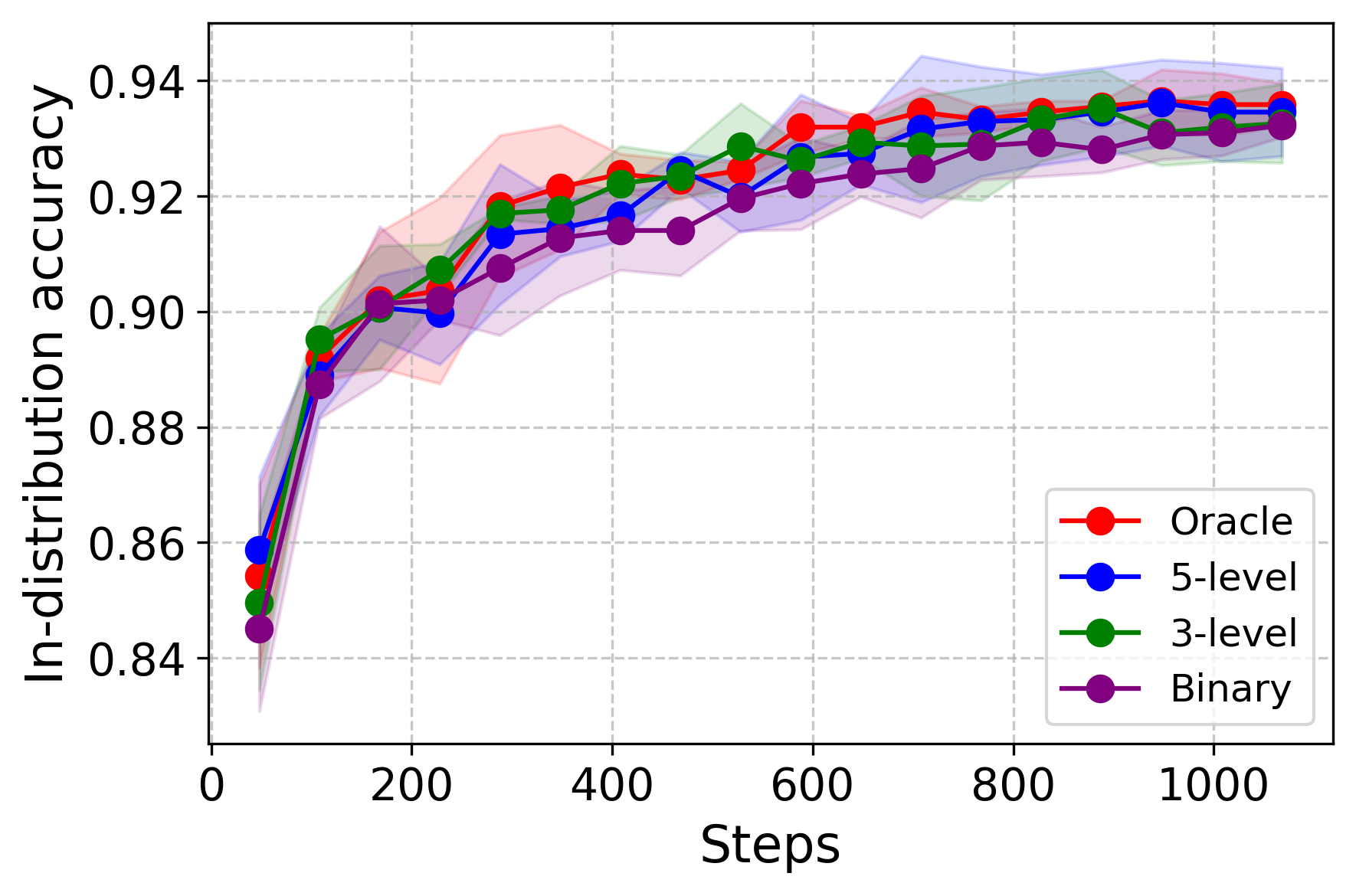}
    \caption{ID accuracy}
    \label{fig:GH_experiment_results_llama_acc}
\end{subfigure}
\hfill
\begin{subfigure}{0.48\textwidth}
    \centering
    \includegraphics[width=\linewidth]{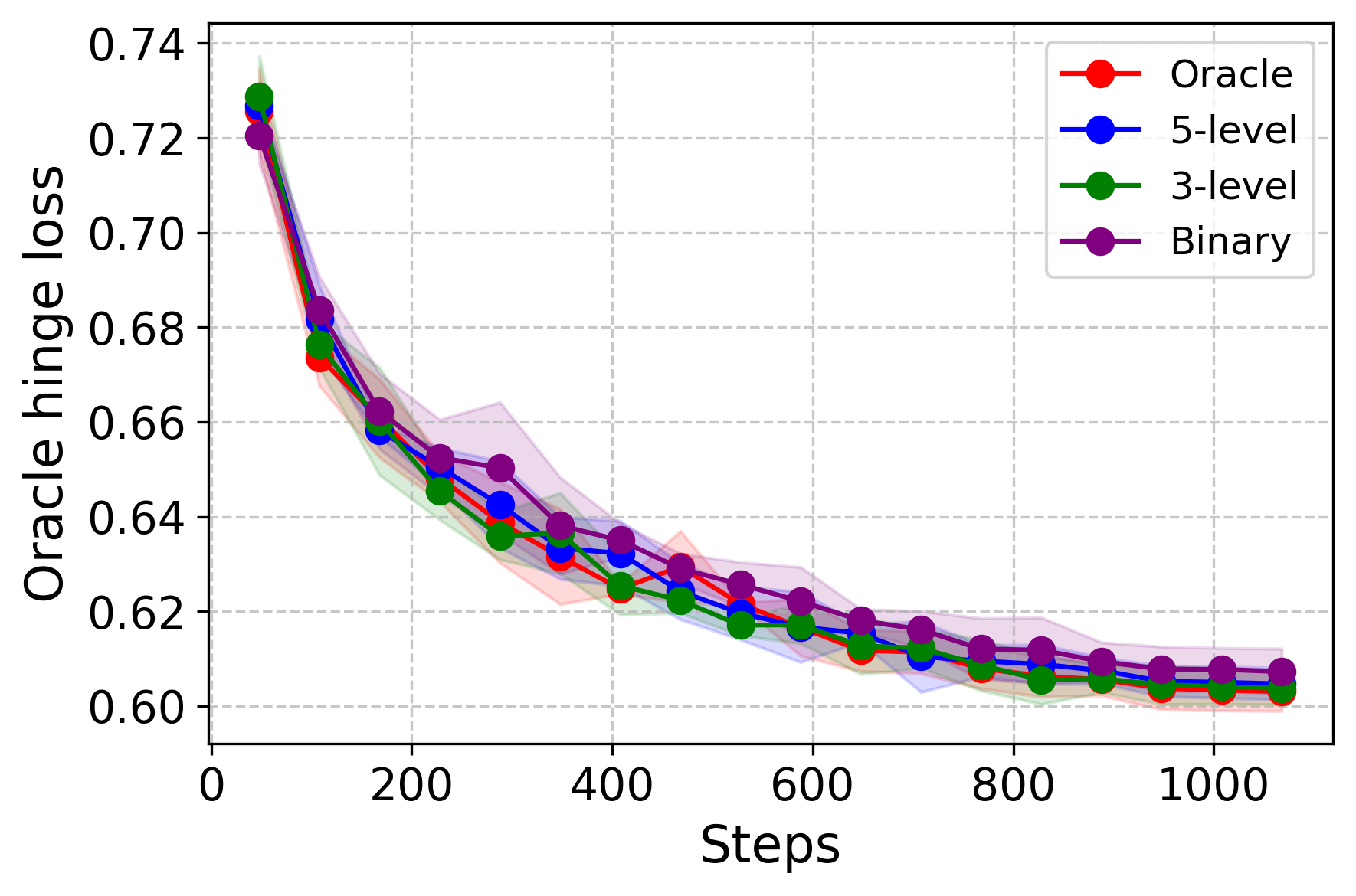}
    \caption{Oracle GH loss}
    \label{fig:GH_experiment_results_gemma_acc}
\end{subfigure}

\caption{The evaluation dynamics of llama models for different ordinal feedback labels under generalized hinge loss.}
\label{fig:GH_experiment_results}
\end{figure}

\begin{table}[ht!]
\centering
\begin{tabular}{cccccccc}
\toprule
\multirow{2}{*}{\text{Model}} & \multirow{2}{*}{\text{Feedback}}& \multicolumn{2}{c}{\text{Oracle CE Loss}} & \multicolumn{2}{c}{\text{ID Accuracy}} & \multicolumn{2}{c}{\text{OOD Accuracy}} \\
\cmidrule(r){3-4} \cmidrule(r){5-6} \cmidrule(r){7-8}
    && Mean & Std & Mean & Std & Mean & Std \\
\midrule
\multirow{4}{*}{\text{Llama}}&Oracle & 0.6030 & 0.004 & 0.9359 & 0.0036 & 0.7798  & 0.0076  \\
&5-level & 0.6046 & 0.0003 & 0.9345 & 0.0076 & 0.7660  &  0.0068  \\
&3-level & 0.6037 & 0.0003 & 0.9326 & 0.0068 & 0.7617  &  0.0163 \\
&Binary & 0.6072 & 0.0005 & 0.9322 & 0.0020 & 0.7580 &  0.0040  \\
\bottomrule
\end{tabular}
\caption{Model convergence statistics for llama under generalized hinge loss.}
\label{tab:GH_experiment_results}
\end{table}

\section{Proofs and Theoretical Discussions}

\subsection{Proof of Theorem \ref{thm:unbiased_interpolation}}
\begin{proof}
We first prove that the constructed ordinal feedback $Z \sim \mu_{j,k}$ satisfies Assumption \ref{assm:wisdom_of_crowd}, where
\[
\mu_{j, k}(z) =
\begin{cases}
(z_k - z_{\text{oracle}})\big/(z_k - z_j), \quad &\text{if }z=z_j,\\
(z_{\text{oracle}} - z_j)\big/(z_k - z_j), \quad &\text{if }z=z_k,\\
0, \quad & \text{otherwise}.
\end{cases}
\]

If $z_{\text{oracle}} = z_i$ for any $z_i \in \mathcal{Z}$, then the constructed measure is the Dirac measure for $z_{\text{oracle}}$, which automatically fulfills the requirement.

We consider the cases where $z_{\text{oracle}} \in (z_j, z_k)$ for some $z_j, z_k \in \mathcal{Z}$. Then the expectation of $Z$ w.r.t. $\mu_{j,k}$ is
\begin{align*}
\mathbb{E}_{Z \sim \mu_{j,k}}[Z] & = \mu_{j,k}(z_j) \cdot z_j + \mu_{j,k}(z_k) \cdot z_k\\
& = (z_j z_k - z_{\text{oracle}} z_j + z_{\text{oracle}} z_k - z_j z_k) \big/ (z_k - z_j) \\
& = z_{\text{oracle}}.
\end{align*}

For the second part of Theorem \ref{thm:unbiased_interpolation}, we prove the conclusion case by case. Suppose we have an ordinal feedback $Z\colon \Omega \rightarrow \mathcal{Z}$ satisfying Assumption \ref{assm:wisdom_of_crowd}. If $z_{\text{oracle}} \in \mathcal{Z}$, for example, $z_{\text{oracle}} = z_{i_0}$, then
\[
\mu_{j,i_0} = \mu_{i_0,k} = \delta_{z_{i_0}},
\]
implying we can set those coefficients $\alpha_{j, i_0}$'s and $\alpha_{i_0, k}$'s to be arbitrary non-negative real numbers such that $\sum_{j}\alpha_{j, i_0} + \sum_{k}\alpha_{i_0, k} = \mu(z_{i_0})$ without affecting the measure on the other points.

Therefore, we can (without loss of generality) assume $z_{\text{oracle}} \notin \mathcal{Z}$. Assume $z_1 < \dots < z_{m_1} < z_{\text{oracle}} < z_{m_1 + 1} < \dots < z_m$.

Case (i): when there is only one point on each side of $z_{\text{oracle}}$, that is, $m_1 = 1$ and $m_1 + 1 = m$.\\
Then by Assumption \ref{assm:wisdom_of_crowd}, we have
\[
\mu(z_1) \cdot z_1 + \mu(z_m) \cdot z_m = z_{\text{oracle}}.
\]
Combining it with the constraint such that
\[
\mu(z_1) + \mu(z_m) = 1,
\]
we have
\[
\mu(z_1) = (z_m - z_{\text{oracle}}) \big/ (z_m - z_1),\quad \mu(z_m) = (z_{\text{oracle}}- z_1) \big/ (z_m - z_1),
\]
which is exactly the same as $\mu_{1, m}$.\\

Case (ii): when there is only one point larger than $z_{\text{oracle}}$, that is, $m_1 + 1 = m$.\\
Then we will prove that there exist non-negative real numbers $\alpha_{j, m}$ such that
\[
\mu = \sum_{j} \ \alpha_{j, m} \ \mu_{j, m}.
\]
In fact, we can construct (for any $j \neq m$)
\[
\alpha_{j, m} \coloneqq \frac{\mu(z_j)}{\mu_{j, m}(z_j)}.
\]
Then each $\alpha_{j, m}$ is non-negative. Furthermore, by Assumption \ref{assm:wisdom_of_crowd},
\begin{align*}
\sum_{j \neq m} \alpha_{j, m} \mu_{j, m}(z_m) & = \sum_{j \neq m} \frac{\mu(z_j)}{\mu_{j,m}(z_j)} \cdot \mu_{j,m}(z_m)\\
& = \frac{\sum_{j \neq m} \mu(z_j) \cdot z_{\text{oracle}} - \sum_{j \neq m} \mu(z_j) \cdot z_j}{z_m - z_{\text{oracle}}}\\
& = \frac{(1-\mu(z_m)) \cdot z_{\text{oracle}} - (z_{\text{oracle}} - \mu(z_m)\cdot z_m)}{z_m - z_{\text{oracle}}} \\
& = \mu(z_m),
\end{align*}
indicating that $\mu = \sum_{j \neq m} \alpha_{j, m} \mu_{j, m}$.\\
By the property of probability measures, we can easily see that
\[
\sum_{j \neq m} \alpha_{j, m} = 1.
\]\\

Case (iii): when there is only one point smaller than $z_{\text{oracle}}$, that is, $m_1 = 1$. This case can be proved similarly to Case (ii).\\

Case (iv): general cases where there are (possibly) multiple points on each side of $z_{\text{oracle}}$. We prove it by induction. Suppose there are $a$ elements in $\mathcal{Z}$ smaller than $z_{\text{oracle}}$ and $b$ elements larger than $z_{\text{oracle}}$. Denote the case by $(a, b)$. Suppose that the conclusion has been proved for all the cases $(a, b)$ if $a < m_1$ or $b < m - m_1$. Now we prove it for the case $(m_1, m - m_1)$. We first define the corresponding index (if there are multiple elements in a tie, select arbitrarily)
\begin{equation}
i_1 \coloneqq \argmin_i |z_i - z_{\text{oracle}}| \cdot \mu(z_i).
\label{eqn:aux_least_weight}
\end{equation}
Without loss of generality, we assume $z_{i_1} < z_{\text{oracle}}$. Then we have $m_1 > 1$ due to Assumption \ref{assm:wisdom_of_crowd}. We now construct a coefficient $\alpha_{i_1, m_1 + 1}$ such that
\[
\alpha_{i_1, m_1 + 1} \coloneqq \frac{\mu(z_{i_1})}{\mu_{i_1, m_1+1}(z_{i_1})}.
\]
By the definition \eqref{eqn:aux_least_weight}, we have
\[
\mu(z_{m_1 + 1}) \geq \alpha_{i_1, m_1 + 1} \cdot \mu_{i_1, m_1 + 1}(z_{m_1 + 1}).
\]
Hence we can construct a new measure
\[
\mu^\prime(z_i) = 
\begin{cases}
0, \quad & \text{if } i = i_1;\\
\big(\mu(z_{m_1 + 1}) - \alpha_{i_1, m_1 + 1} \cdot \mu_{i_1, m_1 + 1}(z_{m_1 + 1})\big) \Big/ (1 - \alpha_{i_1, m_1 + 1}), \quad & \text{if } i = m_1 + 1;\\
\mu(z_i) \big/ (1 - \alpha_{i_1, m_1 + 1}), \quad & \text{otherwise}.
\end{cases}
\]
This measure can be easily verified as a probability measure with $m_1 - 1$ elements smaller than $z_{\text{oracle}}$. By induction hypothesis, we can construct non-negative real numbers $\alpha_{j,k}$'s summing up to 1 such that
\[
\mu^\prime = \sum_{j, k} \alpha_{j, k}^\prime \cdot \mu_{j,k}.
\]
Then we have
\[
\mu = \alpha_{i_1, m_1 + 1} \cdot \mu_{i_1, m_1 + 1} + \sum_{j,k} \frac{\alpha_{j, k}^\prime}{1 - \alpha_{i_1, m_1 + 1}} \cdot \mu_{j, k},
\]
of which the coefficients are non-negative and summing up to 1.
\end{proof}

\subsection{Proof of Proposition \ref{prop:loss_are_affine}}
\label{subapd:loss_affine}
\begin{proof}
\textbf{Cross-entropy loss:}

The cross-entropy loss is
\[
\ell_{\text{ce}}(Z, z) = -\left[Z \log(z) + (1-Z) \log(1-z)\right],
\]
where $z \in [0, 1]$, which is the probability under the BT model. The cross-entropy loss is affine to $Z$:
\allowdisplaybreaks
\begin{align*}
\mathbb{E}_{Z}\big[\ell_{\text{ce}}(Z, z)\big]
& = -\sum_{j, z_j \in \mathcal{Z}} \mathbb{P}(Z = z_j) \cdot \left[z_j \cdot \log(z) + (1-z_j) \cdot \log(1-z)\right]\\
& = -\left[\Big(\sum_{j, z_j \in \mathcal{Z}} \mathbb{P}(Z = z_j) \cdot z_j\Big) \cdot \log(z) + \bigg(1 - \Big(\sum_{j, z_j \in \mathcal{Z}} \mathbb{P}(Z = z_j) \cdot z_j\Big)\bigg) \cdot \log(1-z)\right]\\
& = -\left[\mathbb{E}[Z] \cdot \log(z) + \big(1-\mathbb{E}[Z]\big) \cdot \log(1-z)\right]\\
& = \ell_{\text{ce}}\big(\mathbb{E}[Z], z\big).
\end{align*}

\textbf{Hinge loss:}

The hinge loss sets the loss to be
\[
\ell_{\text{hinge}}(Z, z) = \mathbf{1}\{Z=1\} \cdot \max\big(C-z, 0\big) + \mathbf{1}\{Z=0\} \cdot \max\big(C+z, 0\big),
\]
where $z \in \mathbb{R}$, which is the difference of the reward functions $r_\theta(x, y_1) - r_\theta(x, y_2)$. We can generalize the hinge loss as
\[
\ell_{\text{hinge}}(Z, z) \coloneqq Z \cdot \max\big(C-z, 0\big) + (1-Z) \cdot \max\big(C+z, 0\big)
\]
The (generalized) hinge loss is affine to $Z$ by a similar argument to the cross-entropy loss:
\allowdisplaybreaks
\begin{align*}
& \phantom{=} \mathbb{E}_{Z}\big[\ell_{\text{hinge}}(Z, z)\big] \\
& = \sum_{j, z_j \in \mathcal{Z}} \mathbb{P}(Z = z_j) \cdot \left[z_j \cdot \max\big(C+z,0\big) + (1-z_j) \cdot \max\big(C-z,0\big)\right]\\
& = \Big(\sum_{j, z_j \in \mathcal{Z}} \mathbb{P}(Z = z_j) \cdot z_j\Big) \cdot \max\big(C+z,0\big) + \bigg(1 - \Big(\sum_{j, z_j \in \mathcal{Z}} \mathbb{P}(Z = z_j) \cdot z_j\Big)\bigg) \cdot \max\big(C-z,0\big)\\
& = \mathbb{E}[Z] \cdot \max\big(C+z,0\big) + \big(1-\mathbb{E}[Z]\big) \cdot \max\big(C-z,0\big)\\
& = \ell_{\text{hinge}}\big(\mathbb{E}[Z], z\big).
\end{align*}
\end{proof}

\subsection{Proof of Proposition \ref{prop:same_population_loss}}
\label{subapd:same_population_loss}
\begin{proof}
We can prove a stronger conclusion such that for any $(x, y_1, y_2) \in \mathcal{X} \times \mathcal{Y}^2$,
\[
\mathbb{E}_Z\Big[\ell\big(Z, h(x, y_1, y_2)\big)\Big| (x, y_1, y_2)\Big] = \mathbb{E}_{Z^\prime}\Big[\ell\big(Z^\prime, h(x, y_1, y_2)\big)\Big| (x, y_1, y_2)\Big].
\]
Since both $Z$ and $Z^\prime$ satisfy Assumption \ref{assm:wisdom_of_crowd}, we have for any $(x, y_1, y_2) \in \mathcal{X} \times \mathcal{Y}^2$,
\[
\mathbb{E}[Z | (x, y_1, y_2)] = z_{\text{oracle}}(x, y_1, y_2) = \mathbb{E}[Z^\prime | (x, y_1, y_2)].
\]
Then by the affinity condition \eqref{eqn:loss_affine}, we have for any $h \in \mathcal{H}$,
\begin{align*}
& \phantom{=}\mathbb{E}_Z\Big[\ell\big(Z, h(x, y_1, y_2)\big)\Big| (x, y_1, y_2)\Big]\\
& = \ell\big(\mathbb{E}[Z|(x, y_1, y_2)], h(x, y_1, y_2)\big) \\
& = \ell\big(z_{\text{oracle}}(x, y_1, y_2), h(x, y_1, y_2)\big).
\end{align*}
The same arguments also lead to that
\[
\mathbb{E}_{Z^\prime}\Big[\ell\big(Z^\prime, h(x, y_1, y_2)\big)\Big| (x, y_1, y_2)\Big] = \ell\big(z_{\text{oracle}}(x, y_1, y_2), h(x, y_1, y_2)\big),
\]
which concludes the proof.
\end{proof}

\subsection{Proof of Proposition \ref{prop:existence_of_hierarchical_expectation}}

\begin{proof}
We construct $(W, W^\prime)$ as follows: we set $W$ to be identically distributed as $Z$, and
\[
\mathbb{P}_0(W^\prime = z_k^\prime | W = z_j) \coloneqq \beta_{j, k}.
\]
Then by property (a) of $\beta_{j, k}$'s, we have
\[
\mathbb{E}[W^\prime|W = z_j] = z_j,
\]
indicating
\[
W = \mathbb{E}[W^\prime | W].
\]
By property (b) of $\beta_{j, k}$'s, we have that the constructed $W^\prime$ has a marginal distribution identical to $Z^\prime$. Hence, $(W, W^\prime)$ is a coupling satisfying the hierarchical expectation requirements.\\

On the other hand, if we have a coupling $(W, W^\prime)$ satisfying the hierarchical expectation condition, we can easily verify that the conditional probabilities satisfy the requirements in Proposition \ref{prop:existence_of_hierarchical_expectation}.
\end{proof}

\subsection{Proof of Corollary \ref{corol:hierarchical_expectation_examples}}
\begin{proof}
For case (a) such that $Z = z_{\text{oracle}} \eqqcolon z_1$ almost surely, we have $\mu = \delta_{z_{\text{oracle}}}$. By Assumption \ref{assm:wisdom_of_crowd}, the probability measure $\mu^\prime$ of $Z^\prime$ satisfies
\[
\sum_{k, z_k^\prime \in \mathcal{Z}^\prime} \mu^\prime(z_k^\prime) \ z_k^\prime = z_{\text{oracle}}
\]
almost surely. Setting $\beta_{1,k} \coloneqq \mu^\prime(z_k^\prime)$ fulfills the properties in Proposition \ref{prop:existence_of_hierarchical_expectation}. \\

For case (b) such that $\mathcal{Z}^\prime = \{z_1^\prime \coloneqq 0, z_2^\prime \coloneqq 1\}$, we can construct $\beta_{j, k}$'s as
\[
\beta_{j, 1} = 1 - z_j, \quad \beta_{j, 2} = z_j,
\]
for any $z_j \in \mathcal{Z}$. Then one can easily verify that the construction satisfies the requirement (a) in Proposition \ref{prop:existence_of_hierarchical_expectation}. For part (b), by Assumption \ref{assm:wisdom_of_crowd}, we have
\[
\sum_{j, z_j \in \mathcal{Z}} \alpha_j z_j = z_{\text{oracle}},
\]
and
\[
\alpha_2^\prime = z_{\text{oracle}}.
\]
Combining the above two equalities, we have
\[
\sum_{j, z_j \in \mathcal{Z}} \alpha_j \beta_{j, 2} =
\alpha_2^\prime.
\]
On the other hand,
\[
\sum_{j, z_j \in \mathcal{Z}} \alpha_j (1 - z_j) = 1 - z_{\text{oracle}},
\]
and
\[
\alpha_1^\prime = 1 - z_{\text{oracle}},
\]
which implies that
\[
\sum_{j, z_j \in \mathcal{Z}} \alpha_j \beta_{j, 1} =
\alpha_1^\prime.
\]
\end{proof}

\subsection{Proof of Theorem \ref{thm:Rademacher_complexity}}
\begin{lemma}
\label{lemma:affine_is_cvx}
Any affine function is also convex.
\end{lemma}

\begin{lemma}
\label{lemma:sup_cvx_cvx}
The pointwise supremum of a family of convex functions is still convex. In other words, for any family of convex functions $f_s(\cdot)$ where $s \in S$, we have
$
\sup_{s \in S} f_s(\cdot)
$
still being convex.
\end{lemma}
We provide those two lemmas without proof since the proof can be found in any convex analysis textbook.
A corollary is the following.
\begin{lemma}
\label{lemma:sup_loss_cvx}
If the loss function satisfies the affinity condition \eqref{eqn:loss_affine}, then for any hypothesis class $\mathcal{H}$, the following function
\[
\sup_{h \in \mathcal{H}} \ \sum_{i=1}^n \varepsilon_i \ell\big(\cdot , h(x_i, y_{i, 1}, y_{i, 2})\big)
\]
is convex for any realization of $\varepsilon_i$ taking values in $\{+1, -1\}$.
\end{lemma}
\begin{proof}[Proof of Lemma \ref{lemma:sup_loss_cvx}]
The argument breaks up into three pieces: first, any linear combination of affine functions is still affine (which is straightforward), hence
\[
\sum_{i=1}^n \varepsilon_i \ell\big(\cdot , h(x_i, y_{i, 1}, y_{i, 2})\big)
\]
is affine as long as condition \eqref{eqn:loss_affine} holds.

Second, any affine function is also convex (Lemma \ref{lemma:affine_is_cvx}), therefore,
\[
\sum_{i=1}^n \varepsilon_i \ell\big(\cdot , h(x_i, y_{i, 1}, y_{i, 2})\big)
\]
is also convex for any $h \in \mathcal{H}$.

Third, the supreme of any class of convex functions is still convex (Lemma \ref{lemma:sup_cvx_cvx}), which means taking the supreme over the hypothesis class $\mathcal{H}$ suffices.
\end{proof}

\begin{proof}[Proof of Theorem \ref{thm:Rademacher_complexity}]
We denote the hierarchical expectation coupling of $Z$ and $Z^\prime$ by $(W, W^\prime)$. By direct inspection, we have
\allowdisplaybreaks
\begin{align*}
\mathrm{Rad}_{\mathcal{Z}^\prime, n}(\ell \circ \mathcal{H}) & = \frac{1}{n}\mathbb{E}_{x, y, Z^\prime, \varepsilon} \left[\sup_{h \in \mathcal{H}} \sum_{i=1}^n \varepsilon_i \ell\big(Z_i^\prime, h(x_i, y_{i, 1}, y_{i, 2})\big)\right]\\
& = \frac{1}{n}\mathbb{E}_{x, y, W, W^\prime, \varepsilon} \left[\sup_{h \in \mathcal{H}} \sum_{i=1}^n \varepsilon_i \ell\big(W_i^\prime, h(x_i, y_{i, 1}, y_{i, 2})\big)\right]\\
& = \frac{1}{n}\mathbb{E}_{x, y, \varepsilon} \left[\mathbb{E}_{W, W^\prime}\left[\sup_{h \in \mathcal{H}} \sum_{i=1}^n \varepsilon_i \ell\big(W_i^\prime, h(x_i, y_{i, 1}, y_{i, 2})\big)\right]\right]\\
& = \frac{1}{n}\mathbb{E}_{x, y, \varepsilon} \left[\mathbb{E}_{W}\left[\mathbb{E}_{W^\prime}\left[\sup_{h \in \mathcal{H}} \sum_{i=1}^n \varepsilon_i \ell\big(W_i^\prime, h(x_i, y_{i, 1}, y_{i, 2})\big)\middle| W \right]\right]\right]\\
& \geq \frac{1}{n}\mathbb{E}_{x, y, \varepsilon} \left[\mathbb{E}_{W}\left[\sup_{h \in \mathcal{H}} \sum_{i=1}^n \varepsilon_i \ell\big(\mathbb{E}_{W^\prime}\left[W_i^\prime\middle| W \right], h(x_i, y_{i, 1}, y_{i, 2})\big)\right]\right]\\
& = \frac{1}{n}\mathbb{E}_{x, y, \varepsilon} \left[\mathbb{E}_{W}\left[\sup_{h \in \mathcal{H}} \sum_{i=1}^n \varepsilon_i \ell\big(W, h(x_i, y_{i, 1}, y_{i, 2})\big)\right]\right]\\
& = \frac{1}{n}\mathbb{E}_{x, y, \varepsilon} \left[\mathbb{E}_{Z}\left[\sup_{h \in \mathcal{H}} \sum_{i=1}^n \varepsilon_i \ell\big(Z, h(x_i, y_{i, 1}, y_{i, 2})\big)\right]\right]\\
& = \frac{1}{n}\mathbb{E}_{x, y, Z, \varepsilon} \left[\sup_{h \in \mathcal{H}} \sum_{i=1}^n \varepsilon_i \ell\big(Z, h(x_i, y_{i, 1}, y_{i, 2})\big)\right]\\
& = \mathrm{Rad}_{\mathcal{Z}, n}(\ell \circ \mathcal{H}),
\end{align*}
where the first equality is by definition, the second equality is because the coupling's marginal distribution on $W^\prime$ is equal to that of $Z^\prime$, the third equality is due to the exchangeability of the order of integration (by Fubini's Theorem), the fourth equality is because of the tower property of the conditional expectation, the first inequality is the result of Lemma \ref{lemma:sup_loss_cvx} and Jensen's inequality, the fifth equality is owing to the definition of the hierarchical expectation (Definition \ref{def:hierarchical_expectation}), the sixth is on account of the property of the coupling again, the seventh thanks to Fubini's Theorem again, and the last equality is the definition of Rademacher complexity again.
\end{proof}

\subsection{Proof of Corollary \ref{corol:ordinal_better_than_binary}}
\begin{proof}
A straightforward conclusion of Theorem \ref{thm:Rademacher_complexity} and Corollary \ref{corol:hierarchical_expectation_examples}.
\end{proof}

\subsection{Proof of Theorem \ref{thm:soft_label_benefits}}
\label{subapd:soft_label_benefits}
\begin{proof}
Denote the one-hot vector at dimension $j$ as $e_j$. The following arguments are made for any $x \in \mathcal{X}$, and we omit the dependence on $x$ for notation simplicity.

We only need to show that $\bar{y}_{\mathcal{T}}$ is a hierarchical expectation of $y$ and the theorem is the result of Theorem \ref{thm:Rademacher_complexity}. We construct the coupling $(w, w^\prime)$ as follows:
\[
w \coloneqq \bar{y}_{\mathcal{T}},
\]
and
\[
\mathbb{P}(w^\prime = e_j | w) \coloneqq w_j,
\]
where $w_j$ denotes the $j$-th entry of $w$. We now verify that $(w, w^\prime)$ is a coupling of $(\bar{y}_{\mathcal{T}}, y)$. The fact that $w$ and $\bar{y}_{\mathcal{T}}$ have the same distribution is easy. For $w^\prime$ and $y$, we have
\begin{align*}
\mathbb{P}(w^\prime = e_j) & = \mathbb{E}[\mathbb{P}(w^\prime = e_j | w)]\\
& = \mathbb{E}[w_j] \\
& = \mathbb{E}[\bar{y}_{\mathcal{T}, j}] \\
& = y_{\text{oracle}, j} \\
& = \mathbb{P}(y = e_j),
\end{align*}
where $\bar{y}_{\mathcal{T}, j}$ (or $y_{\text{oracle}, j}$) denotes the $j$-th entry of $\bar{y}_{\mathcal{T}}$ (or $y_{\text{oracle}}$), and the dependence on $x$ has been omitted. Here, the first equality is the tower property of the conditional expectation, the second due to the definition of $w^\prime$, the third because of the construction of $w$, the fourth on account of Assumption \ref{assm:marginal_unbiased_teacher}, and the last is the definition of $y_{\text{oracle}}$.

Hence $(w, w^\prime)$ is a coupling of $(\bar{y}_{\mathcal{T}}, y)$.

We combine that conclusion with the fact that
\[
\mathbb{E}[w^\prime | w] = w,
\]
leading to the conclusion that $\bar{y}_{\mathcal{T}}$ is a hierarchical expectation of $y$.
\end{proof}

\subsection{Generalization Bound under Rademacher Complexity}
\label{subapd:gen_bound}
The following proposition is a well-known generalization bound. We present it here only for the completeness of our argument as the proof can be found in any statistical learning lecture notes.

\begin{proposition}[Generalization Bound]
Suppose we have $\mathcal{D}_{\mathcal{Z}}$ as a dataset consisting of $n$ i.i.d. samples. For any hypothesis class $\mathcal{H}$, we have with probability at least $1-\delta$ that for every function $h \in \mathcal{H}$,
\[
\mathbb{E}\Big[\ell\big(Z_i, h(x_i, y_{i, 1}, y_{i, 2})\big)\Big] \leq \hat{\mathbb{E}}_{\mathcal{D}_{\mathcal{Z}}}\Big[\ell\big(Z_i, h(x_i, y_{i, 1}, y_{i, 2})\big)\Big] + 2 \mathrm{Rad}_{\mathcal{Z}, n}(\ell \circ \mathcal{H}) + \sqrt{\frac{\log(1/\delta)}{n}}.
\]
\end{proposition}

\section{Implementation Details}

\subsection{Dataset Details}
\label{sec:dataset_details}

\textbf{Skywork-Reward-Preference-80K-v0.2.}
The Skywork-Reward-Preference-80K-v0.2 (SRP) dataset is a curated subset of publicly available preference data, spanning a wide range of knowledge domains. Reward models trained on this dataset have achieved top performance in the Reward Bench benchmark. The released version contains 77,016 samples, with approximately 5,000 overlapping samples removed compared to v0.1. In our experiments, we used Skywork-Reward-Gemma-2-27B-v0.2 to annotate data pairs with oracle scores. 

\textbf{Rationale For Scaling.} Instead of directly using the sigmoid values of the oracle model's score differences, we introduce a scaling parameter $T$ because the raw output scores of the oracle model are highly concentrated, as shown in Figure \ref{fig:raw_oracle_label_distrib}. However, it is generally understood that people often hold diverse opinions on preference samples, meaning real-world preference distributions should not be heavily concentrated near a probability of 1. To investigate, we consider two commonly used preference datasets with fine-grained preference scores, UltraFeedback \citep{cui2023ultrafeedback} and HelpSteer2 \citep{wang2024helpsteer2}. A visualization of their preference ratings is provided in Figure \ref{fig:pref_diff_examples}, where many samples show no strong preference but only weak agreement. Based on these observations, we carefully adjusted the scaling parameter $T$ and chose $T=\frac{20}{3}$, ensuring it produces a peak within the slight agreement interval (approximately 0.6-0.7), as shown in Figure \ref{fig:scaled_oracle_label_distrib}. This choice is further justified in the explanation of the tied sample experiments discussed later.

\begin{figure}[ht!]
\centering
    \begin{subfigure}{0.48\textwidth}
        \centering
        \includegraphics[width=\linewidth]{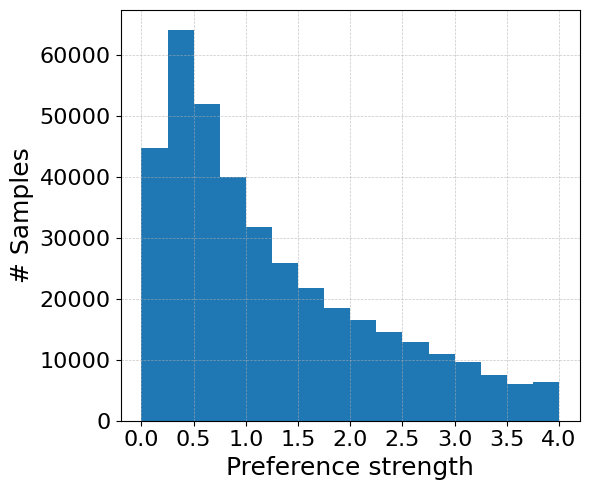}
        \caption{Preference strength distribution in UltraFeedback}
        \label{fig:pref_diff_ultra}
    \end{subfigure}
    \hfill
    \begin{subfigure}{0.48\textwidth}
        \centering
        \includegraphics[width=\linewidth]{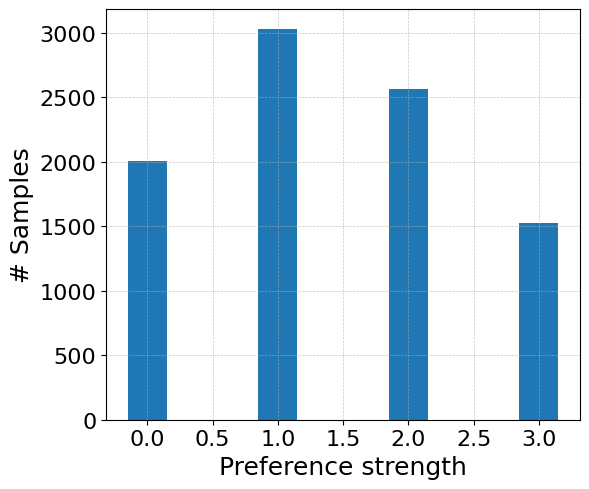}
        \caption{Preference strength distribution in HelpSteer2}
        \label{fig:pref_diff_help}
    \end{subfigure}
\caption{Distributions of preference strengths in the two datasets. For the UltraFeedback dataset, we compare the chosen and rejected scores pairwisely and use their differences as preference strengths. For the HelpSteer2 dataset, its latest version provides a preference strength label and we directly adopt it.}
\label{fig:pref_diff_examples}
\end{figure}

\begin{figure}[ht!]
\centering
    \begin{subfigure}{0.48\textwidth}
        \centering
        \includegraphics[width=\linewidth]{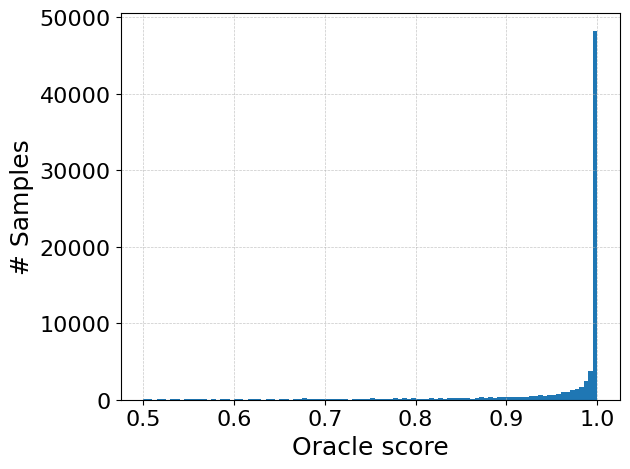}
        \caption{Distribution of raw oracle scores}
        \label{fig:raw_oracle_label_distrib}
    \end{subfigure}
    \hfill
    \begin{subfigure}{0.48\textwidth}
        \centering
        \includegraphics[width=\linewidth]{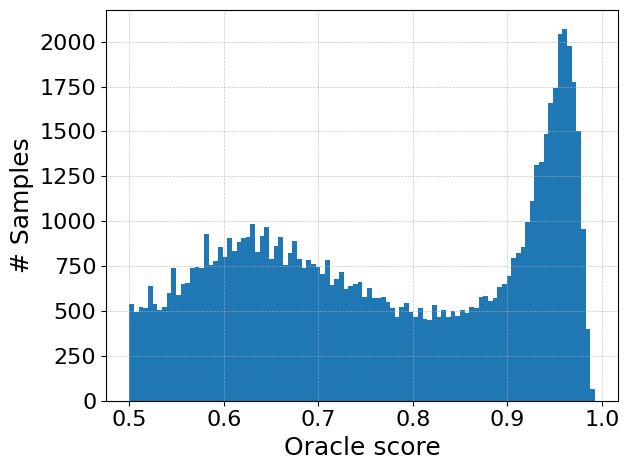}
        \caption{Distribution of scaled oracle scores}
        \label{fig:scaled_oracle_label_distrib}
    \end{subfigure}
\caption{Distribution of oracle labels before and after scaling with $T=\frac{20}{3}$. Note that the oracle score is always recorded for the chosen response relative to the rejected response, hence all oracle scores are no less than 0.5.}
\label{fig:oracle_label_distrib}
\end{figure}

\textbf{Tied Samples.}
As presented in Section \ref{subsec:experiment_setup}, we refer to preference pairs with a label $z = 0.5$ under the 3-level feedback system $\mathcal{Z}_3 = \{0, 0.5, 1\}$ as tied data or tied samples. We use the term ``tied'' because the labels $z$ are generated based on the oracle label $z_{\text{oracle}}$ using an interpolation paradigm, as described in Theorem \ref{thm:unbiased_interpolation}. These labels represent preference samples where people perceive (almost) equal advantages. The detailed sampling process is summarized in Algorithm \ref{alg:sampling}. And it can be easily extended to any ordinal feedback system.

Intuitively, the closer the oracle label of a preference pair is to 0.5, the more likely it is to be assigned a label of $0.5$. An important observation is that the sampled label distribution is dominated by $z_{\text{oracle}}$. To ensure sufficient binary and tied data for the tied ratio experiments, we aim for the binary and tied data to each constitute half of the dataset. After tuning, we found that selecting $T=\frac{20}{3}$ not only simulates a real-world preference data distribution but also satisfies the tied ratio experiment requirements. This further justifies the choice of $T$.

\begin{algorithm}[H]
\caption{3-level Sampling Algorithm}
\label{alg:sampling}
\begin{algorithmic}[1]
\Require $z_{\text{oracle}} \in [0, 1]$ \Comment{Oracle label}
\Ensure $z \in \mathcal{Z}_3 = \{0, 0.5, 1\}$ \Comment{Sampled label}
\If{$z_{\text{oracle}} < 0.5$}
    \State Sample $y \sim \text{Bernoulli}\left(\frac{z_{\text{oracle}}}{0.5}\right)$
    \State $z \gets 0.5 \cdot y$
\ElsIf{$z_{\text{oracle}} > 0.5$}
    \State Sample $y \sim \text{Bernoulli}\left(\frac{z_{\text{oracle}} - 0.5}{0.5}\right)$
    \State $z \gets 0.5 \cdot y + 0.5$
\Else
    \State $z \gets 0.5$
\EndIf
\State \Return $z$
\end{algorithmic}
\end{algorithm}

\textbf{RewardBench Evaluation.}
The RewardBench evaluation dataset combines multiple datasets across four categories: Chat, Chat Hard, Safety, and Reasoning. The scoring follows a standard reward modeling paradigm, where success is defined as the chosen response having a higher score than the rejected response for a given prompt. The evaluation score is computed as a weighted average across all prompts in the selected subset.

\subsection{Training Details}
\label{sec:training_details}

We choose some key training hyperparameters based on grid search. The performance is assessed by in-distribution evaluation (CE) loss under oracle label settings. The grid search space is shown in Table \ref{tab:hyperparam_search}.

\begin{table}[h!]
\centering
\caption{Hyperparameter Search Space}
\begin{tabular}{lc}
\toprule
\textbf{Hyperparameter} & \textbf{Search Range/Values} \\
\midrule
Learning Rate           & [1e-5, 5e-6, 2e-5] \\
Batch Size              & [64, 128] \\
Warm-up Ratio           & [0.03, 0.05, 0.10] \\
\bottomrule
\end{tabular}
\label{tab:hyperparam_search}
\end{table}

The two base models share most of the training parameters, as given in Table \ref{tab:shared_hyperparam}. The different parameters are listed in Table \ref{tab:modelspecified_hyperparam}. 

\begin{table}[h!]
\centering
\caption{Shared Hyperparameters}
\begin{tabular}{lc}
\toprule
\textbf{Hyperparameter} & \textbf{Value} \\
\midrule
Batch Size                  & 128 \\
Optimizer                   & paged\_adamw\_32bit \\
Weight Decay                & 1e-3 \\
Epochs                      & 2 \\
Scheduler                   & Linear Warm-up + Cosine Decay \\
\bottomrule
\end{tabular}
\label{tab:shared_hyperparam}
\end{table}

\begin{table}[h!]
\centering
\caption{Model-specified Hyperparameters}
\begin{tabular}{lcc}
\toprule
\textbf{Hyperparameter} & \textbf{Llama-3.2-1b} & \textbf{Gemma-2-2b} \\
\midrule
Learning Rate               & 1e-5 & 5e-6 \\
Warm-up Ratio               & 0.1 & 0.05 \\
\bottomrule
\end{tabular}
\label{tab:modelspecified_hyperparam}
\end{table}

\end{document}